\documentclass[11pt,journal,onecolumn,twoside]{IEEEtran}

\usepackage{url}
\usepackage{amsmath,epsfig,citesort,cite,multirow}
\usepackage{graphicx,setspace}
\usepackage{url}
\usepackage{amssymb,amsthm}
\usepackage{algorithm}
\usepackage{algorithmic}
 \usepackage{tikz} 
 \usetikzlibrary{arrows,automata}
\usepackage{citesort}
\usepackage{enumerate}
\usepackage{amsxtra}
\usepackage{epsf}
\usepackage{subfigure}
\usepackage{psfrag}
\usepackage{color}
\usepackage{array}
\usepackage{caption}

\usepackage{hyperref}

\usepackage{filehook}

\def\N{\mathcal{N}}

\newcommand{\dimv}{\mathrm{dim}}

\newcommand{\x}{\mathbf x}
\newcommand{\y}{\mathbf y}

\def\S{\mathcal{S}}
\def\Bset{\mathcal{B}}
\def\I{\mathbf{I}}

\def \E {\mathcal{E}}

\newcommand{\Id}{\mathbf{I}}
\newcommand{\1}{\mathbf{1}}
\newcommand{\BB}{\mathbb B} 
\newcommand{\GG}{\mathfrak G} 
\newcommand{\G}{\mathcal G} 
\newcommand{\ngroups}{M}

\DeclareMathOperator{\supp}{supp}

\def\x{{\mathbf x}}

\newtheorem{prob}{Problem}
\newtheorem{rem}{Remark}

\newcommand{\bigO}{\mathcal O}

\newcommand{\Real}{\mathbb R}

\newcommand{\A}{\mathbf{A}}

\usepackage{comment, epsf, epsfig, amssymb, amsmath, graphics, times, cite, mathrsfs, latexsym, graphicx, url, color}

\newtheorem{cor}{Corollary}
\newtheorem{prop}{Proposition}
\newtheorem{definition}{Definition}

\renewcommand{\b}{\mathbf b}
\newtheorem{theorem}{Theorem}
\newtheorem{lemma}{Lemma}
\newcommand{\bitem}{\begin{itemize}}
\newcommand{\eitem}{\end{itemize}}

\DeclareMathOperator*{\argmax}{argmax}        
\DeclareMathOperator*{\argmin}{argmin}

\newcommand{\beqn}{\begin{equation}}
\newcommand{\eeqn}{\end{equation}}
\newcommand{\balign}{\begin{align}}
\newcommand{\ealign}{\end{align}}

\def \dim {N}      

\begin{document}
\title{Group-Sparse Model Selection: Hardness and Relaxations} 

\author{Luca Baldassarre, Nirav Bhan, Volkan Cevher, Anastasios Kyrillidis and Siddhartha Satpathi\thanks{This work was supported in part by the European Commission under Grant MIRG-268398, ERC Future Proof and SNF 200021-132548.} \thanks{LB and VC are with LIONS, EPFL, Lausanne, Switzerland (\{luca.baldassarre, volkan.cevher\}@epfl.ch); NB is with LIDS, MIT (niravb@mit.edu); AK is with the WNCG group at University of Texas at Austin (anastasios@utexas.edu) and SS is at IIT Kharagpur (sidd.piku@gmail.com).}\thanks{Authors are listed in alphabetical order}}

\maketitle


\begin{abstract}
Group-based sparsity models are proven instrumental in linear regression problems for recovering signals from much fewer measurements than standard compressive sensing. 
The main promise of these models is the recovery of ``interpretable" signals through the identification of their constituent groups.  
In this paper, we establish a combinatorial framework for group-model selection problems and highlight the underlying tractability issues. 
In particular, we show that the group-model selection problem is equivalent to the well-known NP-hard weighted maximum coverage problem (WMC).
Leveraging a graph-based understanding of group models, we describe group structures which enable correct model selection in polynomial time via dynamic programming.
Furthermore, group structures that lead to totally unimodular constraints have tractable discrete as well as convex relaxations. 
We also present a generalization of the group-model that allows for within group sparsity, which can be used to model hierarchical sparsity.
Finally, we study the Pareto frontier of group-sparse approximations for two tractable models, among which the tree sparsity model, and illustrate selection and computation trade-offs between our framework and the existing convex relaxations. 

\end{abstract}




\section{Introduction}

\IEEEPARstart{I}{nformation} in many natural and man-made signals can be exactly represented or well approximated by a sparse set of nonzero coefficients in an appropriate basis \cite{mallat1999wavelet}. Compressive sensing (CS) exploits this fact to recover signals from their compressive samples, which are dimensionality reducing, non-adaptive random measurements. According to the CS theory, the number of measurements for stable recovery is proportional to the signal sparsity, rather than to its Fourier bandwidth as dictated by the Shannon/Nyquist theorem \cite{donoho2006compressed, candes2006compressive, baraniuk2007compressive}. Unsurprisingly, the utility of sparse representations also goes well-beyond CS and permeates a lot of fundamental problems in signal processing, machine learning, and theoretical computer science.


Recent results in CS extend the simple sparsity idea to consider more sophisticated {\em structured} sparsity models, which describe the interdependency between the nonzero coefficients  \cite{eldar2009robust, blumensath2009sampling, baraniuk2010model, rao2012signal}. There are several compelling reasons for such extensions:
The structured sparsity models allow to significantly reduce the number of required measurements for perfect recovery in the noiseless case and be more stable in the presence of noise. Furthermore, they facilitate the interpretation of the signals in terms of the chosen structures, revealing information that could be used to better understand their properties. 

An important class of  structured sparsity models is based on groups of variables that should either be selected or discarded together \cite{baraniuk2010low, jenatton2011structured, obozinski2011group, rao2012signal, rao2011convex}.
These structures naturally arise in applications such as neuroimaging \cite{gramfort2009improving,jenatton2011multiscale}, gene expression data \cite{subramanian2005gene,obozinski2011group}, bioinformatics \cite{rapaport2008classification, zhou2010association} and computer vision \cite{cevher2009sparse,baraniuk2010model}. 
For example, in cancer research, the groups might represent genetic pathways that constitute cellular processes. 
Identifying which processes lead to the development of a tumor can allow biologists to directly target certain groups of genes instead of others \cite{subramanian2005gene}. 
Incorrect identification of the active/inactive groups can thus have a rather dramatic effect on the speed at which cancer therapies are developed.

In this paper, we consider {\em group-based} sparsity models, denoted as $\GG$. These structured sparsity models feature collections of groups of variables that could overlap arbitrarily, that is $\GG = \{\G_1, \ldots, \G_M\}$ where each $\G_j$ is a subset of the index set $\{1, \ldots, N\}$, with $N$ being the dimensionality of the signal that we model. Arbitrary overlaps mean that we do not restrict the intersection between any two sets from $\GG$. 

We address the {\em signal approximation}, or projection, problem based on a known group structure $\GG$. That is, given a signal ${\bf x} \in \Real^N$, we seek an $\hat{\bf x}$ closest to it in the Euclidean sense, whose {\em support} (i.e., the  index set of its non-zero coefficients) consists of the union of at most $G$ groups from $\GG$, where $G > 0$ is a user-defined group budget:
$$
\hat{\bf x} \in \argmin\limits_{{\bf z} \in \Real^N} \left \{ \|{\bf x} - {\bf z}\|_2^2 : \supp({\bf z}) \subseteq \bigcup_{\G \in \mathcal{S}} \G, \mathcal{S} \subseteq \GG, |\mathcal{S}| \leq G \right \},
$$
where $\supp({\bf z})$ is the support of the vector ${\bf z}$.  
We call such an approximation as {\em G-group-sparse} or in short {\em group-sparse}. 
The projection problem is a fundamental step in Model-based Iterative Hard-Thresholding algorithms for solving inverse problems by imposing group structures \cite{baraniuk2010model, bah2014model}.

More importantly, we seek to also identify the {\em G-group-support} of the approximation $\hat{{\bf x}}$, that is the $G$ groups that constitute its support. 
We call this the group-sparse {\em model selection} problem.
The G-group-support of $\hat{{\bf x}}$ allows us to ``interpret'' the original signal and discover its properties so that we can, for example, target specific groups of genes instead of others \cite{subramanian2005gene} or focus more precise  imaging techniques on certain brain regions only \cite{michel2011total}.
In this work, we study under which circumstances we can correctly and tractably identify the $G$-group-support of the approximation of a given signal. 
In particular, we show that this problem is equivalent to an NP-hard combinatorial problem known as the weighted maximum coverage problem and we propose a novel polynomial time algorithm for finding its solutions for a certain class of group structures.

If the original signal is affected by noise, i.e., if instead of ${\bf x}$, we measure ${\bf z} := {\bf x} + \boldsymbol{\varepsilon}$, where $\boldsymbol{\varepsilon}$ is some random noise, the $G$-group support of $\hat{\bf z}$ may not exactly correspond to the one of $\hat{\bf x}$. 
Although this is a paramount statistical issue, here we are solely concerned with the computational problem of finding the $G$-group support of a given signal, irrespective of whether it is affected by noise or not, because any group-based interpretation would necessarily require such computation.

{\bf Previous work.} 
Recent works in compressive sensing and machine learning with group sparsity have mainly focused on leveraging group structures for lowering the number of samples required for recovering signals \cite{stojnic2009reconstruction, eldar2009robust, blumensath2009sampling, baraniuk2010model, rao2012signal, huang2011learning,jacob2009group, obozinski2011group}. While these results have established the importance of group structures, many of these works have not fully addressed model selection.

For the special case of non-overlapping groups, dubbed the block-sparsity model, the problem of model selection does not present computational difficulties and features a well-understood theory \cite{stojnic2009reconstruction}. 
The first convex relaxations for group-sparse approximation \cite{yuan2006model} considered only non-overlapping groups. 
Its extension to overlapping groups \cite{zhao2009composite}, however, selects supports defined as the complement of a union of groups (see also \cite{jenatton2011structured}), which is the opposite of what applications usually require, where groups of variables need to be selected together, instead of discarded.

For overlapping groups, Eldar et al.\ \cite{eldar2009robust} consider the union of subspaces framework and cast the model selection problem as a block-sparse model selection one by duplicating the variables that belong to overlaps between the groups. 
Their uniqueness condition \cite{eldar2009robust}[Prop.~1], however, is infeasible for any group structure with overlaps, because it requires that the subspaces intersect only at the origin, while two subspaces defined by two overlapping groups of variables intersect on a subspace of dimension equal to the number of elements in the overlap.

The recently proposed convex relaxations \cite{jacob2009group, obozinski2011group} for group-sparse approximations select group-supports that consist of union of groups.
However, the group-support recovery conditions in \cite{jacob2009group, obozinski2011group} should be taken with care, because they are defined with respect to a particular subset of group-supports and are not general. As we numerically demonstrate in this paper, the group-supports recovered with these methods might be incorrect.
Furthermore, the required consistency conditions in \cite{jacob2009group, obozinski2011group} are unverifiable {\em a priori}. For instance, they require tuning parameters to be known beforehand to obtain the correct group-support. 

Huang et al.~\cite{huang2011learning} use coding complexity schemes over sets to encode sparsity structures.
They consider linear regression problems where the coding complexity of the support of the solution is constrained to be below a certain value.
Inspired by Orthogonal Matching Pursuit, they then propose a greedy algorithm, named StructOMP, that leverages a block-based approximation to the coding complexity. 
A particular instance of coding schemes, namely graph sparsity, can be used to encode both group and hierarchical sparsity.
Their method only returns an approximation to the original discrete problem, as we illustrate via some numerical experiments. 

Obozinski and Bach \cite{obozinski2012convex} consider a penalty involving the sum of a combinatorial function $F$ and the $\ell_p$ norm. 
In order to derive a convex relaxation of the penalty, they first find its tightest positive homogeneous and convex lower bound, which is $F(\supp(\x))^{\frac{1}{q}}\|\x\|_p$, with $\frac{1}{p} + \frac{1}{q} = 1$. 
They also consider set-cover penalties, based on the weighted set cover of a set. 
Given a set function $F$, the weighted set cover of a set $\mathcal{A}$ is the minimum sum of weights of sets that are required to cover $\mathcal{A}$.
With a proper choice of the set function $F$, the weighted set cover can be shown to correspond to the group $\ell_0$-``norm'' that we define in the following.
They establish that the latent group lasso norm as defined in \cite{jacob2009group} is the tightest convex relaxation of the function $\x \mapsto \|\x\|_p \tilde{F}(\supp(\x))^{\frac{1}{q}}$, where $\tilde{F}(\supp(\x))$ is the weighted set cover of the support of $\x$. 

In this work, we take a completely discrete approach and do not rely on relaxations.

{\bf Contributions.}
This paper is an extended version of a prior submission to the IEEE International Symposium on Information Theory (ISIT), 2013. 
This version contains all the proofs that were previously omitted due to lack of space, refined explanations of the concepts, and provides the full description of the proposed dynamic programming algorithms. 

In stark contrast to the existing literature, we take an explicitly discrete approach to identifying group-supports of signals given a budget constraint on the number of groups. This fresh perspective enables us to show that the group-sparse model selection problem is NP-hard: if we can solve the group model selection problem in general, then we can solve any weighted maximum coverage (WMC) problem instance in polynomial time. However, WMC is known to be NP-Hard \cite{hochbaum1997approximation}.  
Given this connection, we can only hope to characterize a subset of instances which are tractable or find guaranteed and tractable approximations.

We present group structures that lead to computationally tractable problems via dynamic programming. 
We do so by leveraging a graph-based representation of the groups and exploiting properties of the induced graph. 
In particular, we present and describe a novel polynomial-time dynamic program that solves the WMC problem for a group structures whose induced graph is a tree or a forest. This result could indeed be of interest by itself. 

We identify tractable discrete relaxations of the group-sparse model selection problem that lead to efficient algorithms. 
Specifically, we relax the constraint on the number of groups into a penalty term and show that if the remaining group constraints satisfy a property related to the concept of total unimodularity \cite{wolsey1999integer}, then the relaxed problem can be efficiently solved using linear program solvers. 
Furthermore, if the graph induced by the group structure is a tree or a forest, we can solve the relaxed problem in linear time by the sum-product algorithm \cite{bishop2006pattern}. 

We extend the discrete model to incorporate an overall sparsity constraint and allowing to select individual elements from each group, leading to within-group sparsity. Furthermore, we discuss how this extension can be used to model hierarchical relationships between variables. We present a novel polynomial-time dynamic program that solves the hierarchical model selection problem exactly and discuss a tractable discrete relaxation.

We also interpret the implications of our results in the context of other group-based recovery frameworks. For instance, the convex approaches proposed in \cite{eldar2009robust, jacob2009group, obozinski2011group} also relax the discrete constraint on the cardinality of the group support. However, they first need to decompose the approximation into vector atoms whose support consists only of one group and then penalize the norms of these atoms.
It has been observed \cite{obozinski2011group} that these relaxations produce approximations that are group-sparse, but their group-support might include irrelevant groups. We concretely illustrate these cases via Pareto frontier examples on two different group structures.

{\bf Paper structure.}
The paper is organized as follows. In Section 2, we present definitions of group-sparsity and related concepts, while in Section \ref{sec:tract}, we formally define the approximation and model-selection problems and connect them to the WMC problem. We present and analyze discrete relaxations of the WMC in Section \ref{sec:discrete_relax} and consider convex relaxations in Section \ref{sec:convex_relax}.  
In Section \ref{sec:discrete_vs_convex}, we illustrate via a simple example the differences between the original problem and the relaxations. 
The generalized model is introduced and analyzed in Section \ref{sec:generalizations}, while numerical simulations are presented in Section \ref{sec:exp}. 
We conclude the paper with some remarks in Section \ref{sec:conclusions}. The appendices contain the detailed descriptions of the dynamic programs.

\section{\label{sec:basics} Basic Definitions}

Let  $ {\bf x} \in \Real^\dim$ be a vector, with $\dimv(\x) = \dim$, and $\N = \{1, \ldots, \dim\}$ be the ground set of its indices. 
We use $|\mathcal{S}|$ to denote the cardinality of an index set $\mathcal{S}$. 
Given a vector ${\bf x} \in \Real^\dim$ and a set $\mathcal{S}$, we define ${\bf x}_\mathcal{S} \in \Real^{|\mathcal{S}|}$, such that the components of ${\bf x}_\mathcal{S}$ are the components of ${\bf x}$ indexed by $\mathcal{S}$.
We use $\mathbb{B}^N$ to represent the space of $N$-dimensional binary vectors and define $\iota: \Real^N \to \mathbb{B}^N$ to be the indicator function of the nonzero components of a vector in $\Real^N$, i.e., $\iota({\bf x})_i = 1$ if $x_i \neq 0$ and $\iota({\bf x})_i = 0$, otherwise.
We let $\1_N$ to be the $N$-dimensional vector of all ones and $\Id_N$ the $N \times N$ identity matrix. 
The support of ${\bf x}$ is defined by the set-valued function $\supp({\bf x}) = \{i \in \N : x_i \neq 0 \}$.
Note that we normally use bold lowercase letters to indicate vectors and bold uppercase letters to indicate matrices. 

We start with the definition of totally unimodularity, a property of matrices that will turn out to be key for obtaining efficient relaxations of integer linear programs.

\begin{definition}
A {\bf totally unimodular matrix} (TU matrix) is a matrix for which every square non-singular submatrix has determinant equal to $-1$ or $1$. 
\end{definition}

We now define the main building block of group sparse model selection, the group structure.
\begin{definition}
A {\bf group structure} $\mathfrak{G} = \{\G_1, \ldots, \G_\ngroups\}$ is a collection of index sets, named {\em groups}, with $\G_j \subseteq \N$ and $|\G_j| = g_j$ for $ 1 \leq j \leq \ngroups$ and $\bigcup_{\G \in \mathfrak{G}} \G = \N$. 
\end{definition}

We can represent a group structure $\GG$ as a bipartite graph, where on one side we have the $N$ variables nodes and on the other the $M$ group nodes. 
An edge connects a variable node $i$ to a group node $j$ if $i \in \G_j$. 
Fig.~\ref{fig:var_groups} shows an example.
The bi-adjacency matrix $\A^\mathfrak{G} \in \mathbb{B}^{N \times \ngroups}$ of the bipartite graph encodes the group structure,
$$
\label{eq:group_structure}
A^\mathfrak{G}_{ij} = \bigg \{ \begin{array}{lc} 1, & \text{if}~i \in \G_j; \\ 0, & \text{otherwise.} \end{array} \; 
$$

\tikzstyle{vnode}=[circle,draw=black,fill=white,thick, minimum size=6pt, inner sep=0pt]
\tikzstyle{vnodeholder}=[circle,draw=black,fill=white,thick, minimum size=1pt, inner sep=0pt]
\tikzstyle{gnode}=[rectangle,draw=black,fill=white,thick, minimum size=6pt, inner sep=0pt]
\tikzstyle{gnodeholder}=[rectangle,draw=black,fill=white,thick, minimum size=1pt, inner sep=0pt]

\begin{figure}
\centering
\begin{tikzpicture}[-,>=stealth',shorten >=1pt,auto,node distance=1.5cm, semithick]

\node[vnode] (v1) at (0,0) [label=above:$1$] [label=left:variables] {};
\node[vnode] (v2) at (1,0) [label=above:$2$] {};
\node[vnode] (v3) at (2,0) [label=above:$3$] {};
\node[vnode] (v4) at (3,0) [label=above:$4$] {};
\node[vnode] (v5) at (4,0) [label=above:$5$] {};
\node[vnode] (v6) at (5,0) [label=above:$6$] {};
\node[vnode] (v7) at (6,0) [label=above:$7$] {};
\node[vnode] (v8) at (7,0) [label=above:$8$] {};

\node[gnode] (g1) at (0,-1.8) [label=below:$\G_1$] [label=left:groups] {};
\node[gnode] (g2) at (1.4,-1.8) [label=below:$\G_2$] {};
\node[gnode] (g3) at (2.8,-1.8) [label=below:$\G_3$] {};
\node[gnode] (g4) at (4.2,-1.8) [label=below:$\G_4$] {};
\node[gnode] (g5) at (5.6,-1.8) [label=below:$\G_5$] {};
\node[gnode] (g6) at (7,-1.8) [label=below:$\G_6$] {};
	
\draw (v1) to (g1);\draw (g1) to (v1);\draw (v2) to (g2);\draw (g2) to (v2);
\draw (v1) to (g3);\draw (v2) to (g3);\draw (v3) to (g3);\draw (v4) to (g3);\draw (v5) to (g3);\draw (g3) to (v1);\draw (g3) to (v2);\draw (g3) to (v3);\draw (g3) to (v4);\draw (g3) to (v5);
\draw (v3) to (g5); \draw (g5) to (v3);
\draw (v4) to (g4);\draw (v6) to (g4);\draw (g4) to (v4);\draw (g4) to (v6);
\draw (v5) to (g5);\draw (v7) to (g5);\draw (g5) to (v5);\draw (g5) to (v7);
\draw (v6) to (g6);\draw (v7) to (g6);\draw (v8) to (g6);\draw (g6) to (v6);\draw (g6) to (v7);\draw (g6) to (v8);

\end{tikzpicture}
\caption{\label{fig:var_groups} Example of bipartite graph between variables and groups induced by the group structure $\GG^1$, see text for details.}
\end{figure}
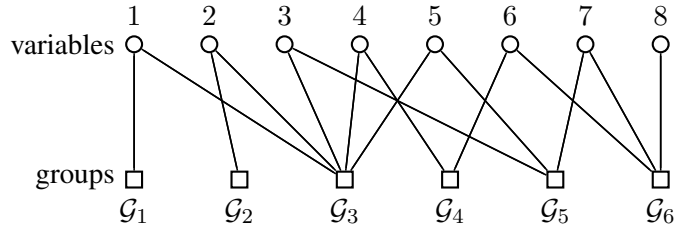

Another useful representation of a group structure is via an {\em intersection graph} $(\mathcal{V}, \mathcal{E})$ where the nodes $\mathcal{V}$ are the groups $\G \in \mathfrak{G}$ and the edge set $\mathcal{E}$ contains $e_{ij}$ if $\G_i \cap \G_j \neq \emptyset$, that is an edge connects two groups that {\em overlap}. 
A sequence of connected nodes $v_1, v_2, \ldots, v_n$, is a {\em cycle} if $v_1 = v_n$.

In order to illustrate these concepts, consider the group structure $\GG^1$ defined by the following groups, $\G_1 = \{1\}$, $\G_2 = \{2\}$, $\G_3 = \{1, 2, 3, 4, 5\}$, $\G_4 = \{4,6\}$, $\G_5 = \{3, 5, 7\}$ and $\G_6 = \{6, 7, 8\}$. 
$\GG^1$ can be represented by the variables-groups bipartite graph of Fig.~\ref{fig:var_groups} or by the intersection graph of Fig.~\ref{fig:bipartite}, which is bipartite and contains cycles.

\tikzstyle{place_black}=[circle,draw=black,fill=black,thick, minimum size=6pt, inner sep=0pt]
\tikzstyle{place_red}=[circle,draw=black,fill=white,thick, minimum size=6pt, inner sep=0pt]
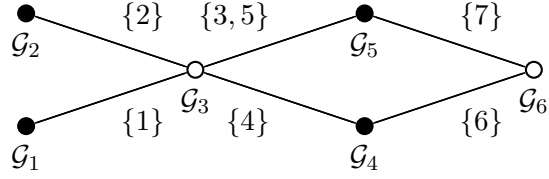
\begin{figure}
\centering
\begin{tikzpicture}[-,>=stealth',shorten >=1pt,auto,node distance=2.8cm, semithick]
  \tikzstyle{every state}=[fill=blue!20,draw=none,text=black]
	\node[place_black] (n1) at (0,0) [label=below:$\G_1$] {};
	\node[place_black] (n2) at (0,1.5) [label=below:$\G_2$] {};
	\node[place_red] (n3) at (2.25,0.75) [label=below:$\G_3$] {};
	\node[place_black] (n4) at (4.5,0) [label=below:$\G_4$] {};
	\node[place_black] (n5) at (4.5,1.5) [label=below:$\G_5$] {};
	\node[place_red] (n6) at (6.75,0.75) [label=below:$\G_6$] {};
  
	\path (n1) edge node {} (n3);
	\path (n2) edge node {$\{2\}$} (n3);
	\path (n3) edge node {} (n4);
	\path (n3) edge node {$\{3, 5\}$} (n5);
	\path (n4) edge node {} (n6);
	\path (n5) edge node {$\{7\}$} (n6);
	
	\path (n3) edge node {$\{1\}$} (n1);
	\path (n3) edge node {} (n2);
	\path (n4) edge node {$\{4\}$} (n3);
	\path (n5) edge node {} (n3);
	\path (n6) edge node {$\{6\}$} (n4);
	\path (n6) edge node {} (n5);
\end{tikzpicture}
\caption{\label{fig:bipartite} Bipartite intersection graph with cycles induced by the group structure $\GG^1$, where on each edge we report the elements of the intersection.}
\end{figure}

An important class of group structures is given by groups whose intersection graph is acyclic (i.e., a tree or a forest) and we call them {\em acyclic group structures}. 
A necessary, but not sufficient, condition for a group structure to have an acyclic intersection graph is that each element of the ground set occurs in at most two groups, i.e., the groups are at most {\em pairwise overlapping}.
Note that a tree or a forest is a bipartite graph, where the two partitions contains the nodes that belong to alternate levels of the tree/forest.
For example, consider $\G_1 = \{1, 2, 3\}$, $\G_2 = \{3, 4, 5\}$, $\G_3 = \{5, 6, 7\}$, which can be represented by the intersection graph in Fig.~\ref{fig:loopless}(Left). 
If $\G_3$ were to include an element from $\G_1$, for example $\{2\}$, we would have the cyclic graph of Fig.~\ref{fig:loopless}(Right).
Note that $\GG^1$ is pairwise overlapping, but not acyclic, since $\G_3, \G_4, \G_5$ and $\G_6$ form a cycle.

\tikzstyle{place}=[circle,draw=black,fill=white,thick, minimum size=6pt, inner sep=0pt]
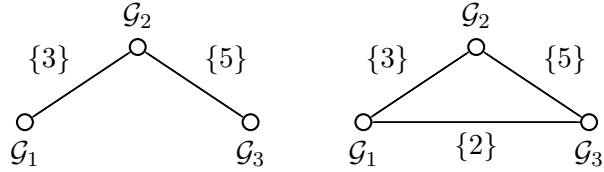
\begin{figure}
\centering
\begin{tikzpicture}[-,>=stealth',shorten >=1pt,auto,node distance=1.5cm, semithick]
	\node[place] (n1) at (0,0)    [label=below:$\G_1$] {};
	\node[place] (n2) at (1.5,1) [label=above:$\G_2$] {};
	\node[place] (n3) at (3,0)    [label=below:$\G_3$] {};
	
\draw (n1) to node {$\{3\}$} (n2);
\draw (n2) to node {} (n1);
\draw (n2) to node {$\{5\}$} (n3);
\draw (n3) to node {} (n2);

	\node[place] (n4) at (4.5,0)    [label=below:$\G_1$] {};
	\node[place] (n5) at (6,1)       [label=above:$\G_2$] {};
	\node[place] (n6) at (7.5,0)    [label=below:$\G_3$] {};
\draw (n4) to node {$\{3\}$} (n5);
\draw (n5) to node {$\{5\}$} (n6);
\draw (n6) to node [label=below:$\{2\}$] {} (n4);

\draw (n5) to node {} (n4);
\draw (n6) to node {} (n5);
\draw (n4) to node {} (n6);

\end{tikzpicture}
\caption{\label{fig:loopless}(Left) Acyclic groups. (Right) By adding one element from $\G_1$ into $\G_3$, we introduce a cycle in the graph.}
\end{figure}

We anchor our analysis of the tractability of interpretability via selection of groups on covering arguments. 
Most of the definition we introduce here can be reformulated as variants of set covers on the support of a signal $\x$, however we believe it is more natural in this context to talk about group covers of a signal $\x$ directly.

\begin{definition}
A {\bf group cover} $\mathcal{S}({\bf x})$ for a signal ${\bf x} \in \Real^N$ is a collection of groups such that $\supp({\bf x}) \subseteq \bigcup_{\G \in \mathcal{S}({\bf x})} \G$. An alternative equivalent definition is given by
$$
\mathcal{S}({\bf x}) = \{\G_j \in \mathfrak{G} : \boldsymbol{\omega} \in \mathbb{B}^\ngroups,~ \omega_j = 1,~ \A^\mathfrak{G}\boldsymbol{\omega} \geq \iota({\bf x})\} \; .
$$
\end{definition}

The binary vector $\boldsymbol{\omega}$ indicates which groups are active and the constraint $\A^\mathfrak{G}\boldsymbol\omega \geq \iota({\bf x})$ makes sure that, for every non-zero component of ${\bf x}$, there is at least one active group that covers it.
We also say that $\mathcal{S}({\bf x})$ {\em covers} ${\bf x}$.
Note that the group cover is often not unique and $\mathcal{S}({\bf x}) = \mathfrak{G}$ is a group cover for any signal ${\bf x}$. This observation leads us to consider more restrictive definitions of group covers.

\begin{definition}
A {\bf $G$-group cover} $\mathcal{S}^G({\bf x}) \subseteq \mathfrak{G}$ is a group cover for ${\bf x}$ with at most $G$ elements,
$$
\label{eq:G_cover}
\mathcal{S}^G({\bf x}) = \{\G_j \in \mathfrak{G} : \boldsymbol{\omega} \in \mathbb{B}^\ngroups,~\omega_j = 1,~\A^\mathfrak{G}\boldsymbol{\omega} \geq \iota({\bf x}),~\sum_{j=1}^\ngroups \omega_j  \leq G\} \; .
$$
\end{definition}

It is not guaranteed that a $G$-group cover always exists for any value of $G$. 
Finding the smallest $G$-group cover lead to the following definitions.

\begin{definition}
The {\bf group $\ell_0$-``norm"} is defined as
\begin{equation}
\label{eq:group_0_norm}
\|{\bf x}\|_{\GG,0} := \min\limits_{\boldsymbol{\omega} \in \mathbb{B}^\ngroups} \left \{ \sum_{j=1}^\ngroups \omega_j : \A^\mathfrak{G}\boldsymbol{\omega} \geq \iota({\bf x}) \right \} \; .
\end{equation}
\end{definition}

\begin{definition}
A {\bf minimal group cover} for a signal ${\bf x} \in \Real^N$ is defined as $\mathcal{M}({\bf x}) = \{\G_j \in \GG : \hat{\omega}({\bf x})_j =1\}$, where 
$\hat{\boldsymbol{\omega}}$ is a minimizer for \eqref{eq:group_0_norm},
$$
\hat{\boldsymbol{\omega}}({\bf x}) \in \argmin\limits_{\boldsymbol{\omega} \in \mathbb{B}^\ngroups} \left \{ \sum_{j=1}^\ngroups \omega_j : \A^\mathfrak{G}\boldsymbol{\omega} \geq \iota({\bf x}) \right \} \; .
$$
\end{definition}

A {\em minimal group cover} $\mathcal{M}({\bf x})$ is a group cover for the support of ${\bf x}$ with minimal cardinality.
Note that there exist pathological cases where for the same group $\ell_0$-``norm", we have different minimal group cover models. 
The minimal group cover can also be seen as the minimum set cover of the support of ${\bf x}$.

\begin{definition}
A signal ${\bf x}$ is {\bf $G$-group sparse} with respect to a group structure $\mathfrak{G}$ if $\|{\bf x}\|_{\GG,0} \leq G$.
\end{definition}

In other words, a signal is {\em $G$-group sparse} if its support is contained in the union of at most $G$ groups from $\GG$.

\section{\label{sec:tract}Tractability of interpretations}
Although real signals may not be exactly group-sparse, it is possible to give a group-based interpretation by finding a group-sparse approximation and identifying the groups that constitute its support.
In this section, we establish the hardness of finding group-based interpretations of signals in general and characterize a class of group structures that lead to tractable interpretations. 
In particular, we present a polynomial time algorithm that finds the correct $G$-group-support of the $G$-group-sparse approximation of ${\bf x}$, given a positive integer $G$ and the group structure $\mathfrak{G}$.

We first define the $G$-group sparse approximation $\hat{{\bf x}}$ and then show that it can be easily obtained from its $G$-group cover $\mathcal{S} ^G(\hat{\bf x})$, which is the solution of the model selection problem. 
We then reformulate the model selection problem as the weighted maximum coverage problem. 
Finally, we present our main result, the polynomial time dynamic program for acyclic group structures. 

\begin{prob}[Signal approximation]
Given a signal ${\bf x} \in \Real^N$, a best $G$-group sparse approximation $\hat{\bf x}$ is given by
\begin{equation}
\label{eq:approx}
\hat{\bf x} \in \argmin\limits_{{\bf z} \in \Real^N} \left \{ \|{\bf x} - {\bf z}\|_2^2 : \|{\bf z}\|_{\GG,0} \leq G \right \}.
\end{equation}
\end{prob}

If we already know the $G$-group cover of the approximation $\mathcal{S}^G(\hat{\bf x})$, we can obtain $\hat{\bf x}$ as $\hat{\bf x}_\mathcal{I} = {\bf x}_\mathcal{I}$ and $\hat{\bf x}_{\mathcal{I}^c} = 0$, where $\mathcal{I} = \bigcup_{\G \in \mathcal{S}^G(\hat{\bf x})} \G$ and $\mathcal{I}^c = \N \setminus \mathcal{I}$.
Therefore, we can solve Problem 1 by solving the following discrete problem.
\begin{prob}[Model selection]
Given a signal ${\bf x} \in \Real^N$, a $G$-group cover model for its $G$-group sparse approximation is expressed as follows
\begin{equation}
\label{eq:model_sel}
\mathcal{S}^G(\hat{\bf x}) \in \argmax\limits_{\scriptsize \mathcal{S} \subseteq \GG} \left \{ \sum\limits_{i \in \mathcal{I}} x_i^2 : \mathcal{I} = \bigcup_{\G \in \mathcal{S}} \G,~|\mathcal{S}| \leq G \right \}.
\end{equation}
\end{prob}
To show the connection between the two problems, we first reformulate Problem 1 as
$$
\min\limits_{{\bf z} \in \Real^N} \left \{ \|{\bf x} - {\bf z}\|_2^2 : \supp({\bf z}) = \mathcal{I}, \mathcal{I} = \bigcup_{\G \in \S} \G, \S \subseteq \GG, |\S| \leq G \right \},
$$
which can be rewritten as
$$
\min\limits_{\scriptsize \begin{array}{c} \S \subseteq \GG\\ |\S| \leq G\\ \mathcal{I} = \bigcup_{\G \in \S} \G\end{array}} \min\limits_{\scriptsize \begin{array}{c} {\bf z} \in \Real^N\\ \supp({\bf z}) = \mathcal{I}\end{array}} \|{\bf x} - {\bf z}\|_2^2 \; .
$$
The optimal solution is not changed if we introduce a constant, change sign of the objective and consider maximization instead of minimization
$$
\max\limits_{\scriptsize \begin{array}{c} \S \subseteq \GG\\ |\S| \leq G\\ \mathcal{I} = \bigcup_{\G \in \S} \G\end{array}} \max\limits_{\scriptsize \begin{array}{c} {\bf z} \in \Real^N\\ \supp({\bf z}) = \mathcal{I}\end{array}} \bigg \{ \|{\bf x}\|_2^2 - \|{\bf x} - {\bf z}\|_2^2 \bigg \} \; .
$$
The internal maximization is achieved for $\hat{\bf x}$ as $\hat{\bf x}_\mathcal{I} = {\bf x}_\mathcal{I}$ and $\hat{\bf x}_{\mathcal{I}^c} = 0$, so that we have, as desired, 
$$
\S^G(\hat{{\bf x}}) \in \argmax\limits_{\scriptsize \begin{array}{c} \S \subseteq \GG\\ |\S| \leq G\\\mathcal{I} = \bigcup_{\G \in \S} \G\end{array}} \|{\bf x}_\mathcal{I}\|_2^2 \; .
$$

The following reformulation of Problem 2 as a binary problem allows us to characterize its tractability.
\begin{lemma}
Given ${\bf x} \in \Real^N$ and a group structure $\GG$, we have that $\S^G(\hat{\bf x}) = \{ \G_j \in \GG : \omega^G_j = 1 \}$, where $(\boldsymbol{\omega}^G, {\bf y}^G)$ is an optimal solution of 
\begin{equation}
\label{eq:WMC}
\max\limits_{\boldsymbol{\omega} \in \mathbb{B}^\ngroups,~{\bf y} \in \mathbb{B}^\dim} \left \{ \sum_{i=1}^N y_i x_i^2 : \A^\mathfrak{G} \boldsymbol{\omega} \geq {\bf y},  \sum_{j=1}^\ngroups \omega_j \leq G \right \}.
\end{equation}
\end{lemma}

\begin{proof}
The proof follows along the same lines as the proof in \cite{kyrillidis2012combinatorial}. Note that in \eqref{eq:WMC}, $\boldsymbol{\omega}$ and $\bf y$ are binary variables that specify which groups and which variables are selected, respectively. The constraint $\A^\mathfrak{G} \boldsymbol{\omega} \geq {\bf y}$ makes sure that for every selected variable at least one group is selected to cover it, while the constraint $\sum_{j=1}^\ngroups \omega_j \leq G$ restricts choosing at most $G$ groups.
\end{proof}

Problem \eqref{eq:WMC} can produce all the instances of the weighted maximum coverage problem (WMC), where the weights for each element are given by $x_i^2$ ($1 \leq i \leq \dim$) and the index sets are given by the groups $\G_j \in \GG$ ($1 \leq j \leq \ngroups$).
Since WMC is in general NP-hard \cite{hochbaum1997approximation} and given Lemma 1, the tractability of \eqref{eq:model_sel} directly depends on the hardness of \eqref{eq:WMC}, which leads to the following result.
\begin{prop}
The model selection problem \eqref{eq:model_sel} is in general NP-hard.
\end{prop}

It is possible to approximate the solution of \eqref{eq:WMC} using the greedy WMC algorithm \cite{nemhauser1978analysis}.
At each iteration, the algorithm selects the group that covers new variables with maximum combined weight until $G$ groups have been selected. 
However, we show next that for certain group structures we can find an exact solution.

Our main result is an algorithm for solving \eqref{eq:WMC} for acyclic group structures. 
The proof is given in Appendix \ref{sec:dp_lp}.

\begin{theorem}
\label{theo:DP1}
Given an acyclic group structure $\mathfrak{G}$, there exists a polynomial time dynamic programming algorithm that solves \eqref{eq:WMC}. 
\end{theorem}

\begin{rem}
Sets that are included in one another can be excluded because choosing the larger set would be a strictly dominant strategy, making the smaller set redundant. 
However, the correctness of the dynamic program is unaffected even if such sets are present, as long as the intersection graph remains acyclic.
\end{rem}

\begin{rem}
It is also possible to consider the case where each group $\G_i$ has a cost $C_i$ and we are given a maximum group cost budget $C$. 
The problem then becomes the Budgeted Maximum Coverage \cite{khuller1999budgeted}. 
However, this problem is NP-hard, even in the non-overlapping case, because it generalizes the knapsack problem. 
However, similarly to the pseudo-polynomial time algorithm for knapsack \cite{kellerer2004knapsack}, we can easily devise a pseudo-polynomial time algorithm for the weighted group sparse problem, even for acyclic overlaps. The only condition is that the costs must be integers. 
The time complexity of the resulting algorithm is then polynomial in $C$, the maximum group cost budget.
The algorithm is almost the same as the one given in Appendix \ref{sec:dp_lp}: instead of keeping track of selecting $g$ groups, where $g$ varies from $1$ to $G$; we keep track of selecting groups with total weight equal to $c$, where $c$ varies from $1$ to $C$.
\end{rem}

\section{Discrete relaxations}
\label{sec:discrete_relax}

Relaxations are useful techniques that allow to obtain approximate, or even sometimes exact, solutions while being computationally less demanding.
In our case, we relax the constraint on the number of groups in \eqref{eq:WMC} into a regularization term with parameter $\lambda > 0$, which amounts to paying a penalty of $\lambda$ for each selected group. 
We then obtain the following binary linear program
\begin{equation}
(\boldsymbol{\omega}^\lambda, {\bf y}^\lambda) \in  \argmax\limits_{\boldsymbol{\omega} \in \BB^\ngroups,~{\bf y} \in \BB^\dim} \left \{ \sum_{i=1}^N y_i x_i^2 - \lambda\sum_{j=1}^\ngroups \omega_j: \A^\mathfrak{G} \boldsymbol{\omega} \geq {\bf y} \right \} \;
\label{eq:PR}
\end{equation}
We can rewrite the previous program in standard form. Let $\mathbf{u}^\top = [{\bf y}^\top~\boldsymbol{\omega}^\top] \in \BB^{\dim+\ngroups}$, $\mathbf{w}^\top = [x_1^2,  \ldots, x_N^2,-\lambda\1_M^\top] \in \mathbb{R}^{\dim + \ngroups}$ and $\mathbf{C} = [\Id_N,~-\A^\GG] \in \BB^{\dim \times (\dim + \ngroups)}$. We then have that \eqref{eq:PR} is equivalent to
\begin{equation}
\mathbf{u}^\lambda \in  \argmax\limits_{\mathbf{u} \in \BB^{\dim+\ngroups}} \left \{ \mathbf{w}^\top \mathbf{u} : \mathbf{Cu} \leq 0 \right \}
\label{eq:PR_std}
\end{equation}

In general, \eqref{eq:PR_std} is NP-hard, however, it is well known \cite{wolsey1999integer} that if the constraint matrix $\mathbf{C}$ is Totally Unimodular (TU), then it can be solved in polynomial-time. 
While the concatenation of two TU matrices is not TU in general, the concatenation of the identity matrix with a TU matrix results in a TU matrix. Thus, due to its structure, $\mathbf{C}$ is TU if and only if $\A^\mathfrak{G}$ is TU \cite[Proposition 2.1]{wolsey1999integer}.

The next lemma characterizes which group structures lead to totally unimodular constraints.

\begin{prop}
Group structures whose intersection graph is bipartite lead to constraint matrices $\A^\mathfrak{G}$ that are TU. 
\end{prop}
\begin{proof}
We first use a result that establishes that if a matrix is TU, then its transpose is also TU \cite[Proposition 2.1]{wolsey1999integer}. 
We then apply \cite[Corollary 2.8]{wolsey1999integer} to $\A^\mathfrak{G}$, swapping the roles of rows and columns. 
Given a $\{0, 1, -1\}$ matrix whose columns can be partitioned into two sets, $\mathcal{S}_1$ and $\mathcal{S}_2$, and with no more than two nonzero elements in each row, this corollary provides two sufficient conditions for it being totally unimodular:
\begin{enumerate}
	\item If two nonzero entries in a row have the same sign, then the column of one is in $\mathcal{S}_1$ and the other is in $\mathcal{S}_2$.
	\item If two nonzero entries in a row have opposite signs, then their columns are both in $\mathcal{S}_1$ or both in $\mathcal{S}_2$.
\end{enumerate}
In our case, the columns of $\A^\mathfrak{G}$, which represent groups, can be partitioned in two sets, $\mathcal{S}_1$ and $\mathcal{S}_2$ because the intersection graph is bipartite. The two sets represents groups which have no common overlap so that each row of $\A^\mathfrak{G}$ contains at most two nonzero entries, one in each set. Furthermore, the entries in $\A^\mathfrak{G}$ are only $0$ or $1$, so that condition 1) is satisfied and condition 2) does not apply.
\end{proof}

\begin{cor}
Acyclic group structures lead to totally unimodular constraints.
\end{cor}
\begin{proof}
Acyclic group structures have an intersection graph which is a tree or a forest, which is bipartite.
\end{proof}

The worst case complexity for solving the linear program \eqref{eq:PR_std}, via a primal-dual method \cite{wright1997primal}, is $\bigO(\dim^2(\dim+\ngroups)^{1.5})$, which is greater than the complexity of the dynamic program of Theorem \ref{theo:DP1}.
However, in practice, using an off-the-shelf LP solver may still be faster, because the empirical performance is usually much better than the worst case complexity.

Another way of solving the linear program for acyclic group structures is to reformulate it as an energy maximization problem over a tree, or forest.
In particular, let $\psi_i = \|\x_{\G_i}\|_2^2$ be the energy captured by group $\G_i$ and $\psi_{ij} = \|\x_{\G_i \cap \G_j}\|_2^2$ the energy that is double counted if both $\G_i$ and $\G_j$ are selected, which then needs to be subtracted from the total energy.
Problem \eqref{eq:PR} can then be formulated as
$$
\max_{\boldsymbol{\omega} \in \BB^\ngroups} \sum_{i=1}^\ngroups (\psi_i - \lambda) \omega_i - \sum_{(i,j) \in \mathcal{E}} \psi_{ij} \omega_i \omega_j \; .
$$
This problem is equivalent to finding the most probable state of the binary variables $\omega_i$, where their probabilities can be factored into node and edge potentials. 
These potentials can be computed in $\bigO(N)$ time via a single sweep over the elements, then the most probable state can be exactly estimated by the max-sum algorithm in only $\bigO(M)$ operations, by sending messages from the leaves to the root and then propagating other message from  the root back to the leaves \cite{bishop2006pattern}.

%
The next lemma establishes when the regularized solution coincides with the solution of \eqref{eq:WMC}.
\begin{lemma}
If the value of the regularization parameter $\lambda$ is such that the solution $(\boldsymbol{\omega}^\lambda, {\bf y}^\lambda)$ of \eqref{eq:PR} satisfies $\sum_j \omega_j^\lambda = G$, then $(\boldsymbol{\omega}^\lambda, {\bf y}^\lambda)$ is also a solution for \eqref{eq:WMC}. 
\end{lemma}
\begin{proof}
This lemma is a direct consequence of Prop. \ref{prop:pareto_optimal} below.
\end{proof}

However, as we numerically  show in Section \ref{sec:exp}, given a value of $G$ it is not always possible to find a value of $\lambda$ such that the solution of \eqref{eq:PR} is also a solution for \eqref{eq:WMC}. 
Let the set of points $\mathcal{P} = \{G, (f(G))\}_{G=1}^\ngroups$, where $f(G) = \sum_{i=1}^N y^G_i x_i^2$, be the Pareto frontier of \eqref{eq:WMC}.
We then have the following characterization of the solutions of the discrete relaxation.

\begin{prop}
\label{prop:pareto_optimal}
The discrete relaxation \eqref{eq:PR} yields only the solutions that lie on the intersection between the Pareto frontier of \eqref{eq:WMC}, $\mathcal{P}, $ and the boundary of the convex hull of $\mathcal{P}$.
\end{prop}
\begin{proof}
The solutions of \eqref{eq:WMC} for all possible values of $G$ are the Pareto optimal solutions \cite[Section 4.7]{boyd2004convex} of the following vector-valued minimization problem with respect to the positive orthant of $\Real^2$, which we denote by $\Real^2_+$,
\begin{equation}
\label{eq:vv_wmc}
\begin{array}{cc}
\min\limits_{\boldsymbol{\omega} \in \mathbb{B}^\ngroups,~{\bf y} \in \mathbb{B}^\dim} & {\bf f}(\boldsymbol{\omega},{\bf y}) \\
\text{subject to} & \A^\mathfrak{G} \boldsymbol{\omega} \geq {\bf y}
\end{array}
\end{equation}
where ${\bf f}(\boldsymbol{\omega},{\bf y}) = \left (\|\x\|^2 - \sum_{i=1}^N y_i x_i^2, \sum_{j=1}^\ngroups \omega_j \right ) \in \Real^2_+$. Specifically, the two components of the vector-valued function ${\bf f}$ are the approximation error $E$, and the number of groups $G$ that cover the approximation. It is not possible to simultaneously minimize both components, because they are somehow adversarial: unless there is a group in the group structure that covers the entire support of $\x$, lowering the approximation error requires selecting more groups. Then there exist the so called Pareto frontier of the vector-valued optimization problem defined by the points $(E_G, G)$ for each choice of $G$, i.e.~the second component of ${\bf f}$, where $E_G$ is the minimum approximation error achievable with a support covered by at most $G$ groups.

The scalarization of \eqref{eq:vv_wmc} yields the following discrete problem, with $\lambda > 0$
\begin{equation}
\label{eq:scalar_wmc}
\begin{array}{cc}
\min\limits_{\boldsymbol{\omega} \in \mathbb{B}^\ngroups,~{\bf y} \in \mathbb{B}^\dim} & \|{\bf x}\|^2 - \sum_{i=1}^N y_i x_i^2 + \lambda \sum_{j=1}^\ngroups \omega_j \\
\text{subject~to} & \A^\mathfrak{G} \boldsymbol{\omega} \geq {\bf y}
\end{array}
\end{equation}
whose solutions are the same as for \eqref{eq:PR}. 
Therefore, the relationship between the solutions of \eqref{eq:WMC} and \eqref{eq:PR} can be inferred by the relationship between the solutions of \eqref{eq:vv_wmc} and \eqref{eq:scalar_wmc}.
It is known that the solutions of \eqref{eq:scalar_wmc} are also Pareto optimal solutions of \eqref{eq:vv_wmc}, but only the Pareto optimal solutions of \eqref{eq:vv_wmc} that admit a supporting hyperplane for the feasible objective values of \eqref{eq:vv_wmc} are also solutions of \eqref{eq:scalar_wmc} \cite[Section 4.7]{boyd2004convex}.
In other words, the solutions obtainable via scalarization belong to the intersection of the Pareto optimal solution set and the boundary of its convex hull. 
\end{proof}

\section{Convex relaxations}
\label{sec:convex_relax}

For tractability and analysis, convex proxies to the group $\ell_0$-norm have been proposed (e.g., \cite{jacob2009group}) for finding  group-sparse approximations of signals. 
Given a group structure $\GG$, an example generalization is defined as
\begin{equation}
\label{eq:atomic_norm}
\|{\bf x}\|_{\GG,\{1,p\}} := \inf\limits_{{\scriptsize \begin{array}{c} \mathbf{v}^1, \ldots, \mathbf{v}^\ngroups \in \Real^N\\\forall j, \supp(\mathbf{v}^j) = \G_j \end{array}}} \left \{ \sum_{j = 1}^\ngroups d_j \|\mathbf{v}^j\|_p: \sum_{j=1}^\ngroups \mathbf{v}^j = {\bf x} \right \} ,
\end{equation}
where $\|{\bf x}\|_p = \left ( \sum_{i=1}^N x_i^p \right )^{1/p}$ is the $\ell_p$-norm, and $d_j$ are positive weights that can be designed to favor certain groups over others \cite{obozinski2011group}. 
This norm, also called Latent Group Lasso norm in the literature, can be seen as a weighted instance of the atomic norm described in \cite{rao2012signal}, where the authors leverage convex optimization for signal recovery, but not for model selection.

One can in general use \eqref{eq:atomic_norm} to find a group-sparse approximation under the chosen group norm
\begin{equation}
\label{eq:latent_gl}
\hat{\bf x} \in \argmin\limits_{{\bf z} \in \Real^N} \left \{ \|{\bf x} - {\bf z}\|_2^2 : \|{\bf z}\|_{\GG,\{1,p\}} \leq \lambda\right \}, \;
\end{equation}
where $\lambda > 0$ controls the trade-off between approximation accuracy and group-sparsity. 
However, solving \eqref{eq:latent_gl} does not yield a group-support for $\hat{\bf x}$:
even though we can recover one through the decomposition $\{ \mathbf{v}^j\}$ used to compute $\|\hat{\bf x}\|_{\GG, \{1,p\}}$, it may not be unique as observed in \cite{obozinski2011group} for $p = 2$.
In order to characterize the group-support for ${\bf x}$ induced by \eqref{eq:atomic_norm}, in \cite{obozinski2011group} the authors define two group-supports for $p = 2$. 
The {\em strong group-support} $\breve{\mathcal{S}}({\bf x})$ contains the groups that constitute the supports of each decomposition used for computing \eqref{eq:atomic_norm}. 
The {\em weak group-support} $\mathcal{S}({\bf x})$ is defined using a dual-characterisation of the group norm \eqref{eq:atomic_norm}. 
If $\breve{\mathcal{S}}({\bf x}) = \mathcal{S}({\bf x})$, the group-support is uniquely defined. 
However, \cite{obozinski2011group} observed that for some group structures and signals, even when $\breve{\mathcal{S}}({\bf x}) = \mathcal{S}({\bf x})$, the group-support does not capture the minimal group-cover of ${\bf x}$. 
Hence, the equivalence of $\ell_0$ ``norm'' and $\ell_1$ norm minimization \cite{donoho2006compressed, candes2006compressive} in the standard compressive sensing setting does not hold in the group-based sparsity setting. 
Therefore, even for acyclic group structures, for which we can obtain exact identification of the group support of the approximations via dynamic programming, the convex relaxations are not guaranteed to find the correct group support. 
We illustrate this case via a simple example in the next section.
It remains an open problem to characterize which classes of group structures and signals admit an exact identification via convex relaxations.

\section{Case study: discrete vs.\ convex interpretability}
\label{sec:discrete_vs_convex}

The following stylized example illustrates situations that can potentially be encountered in practice. 
In these cases, the group-support obtained by the convex relaxation will not coincide with the discrete definition of group-cover, while the dynamical programming algorithm of Theorem \ref{theo:DP1} is able to recover the correct group-cover.

Let $\N = \{1, \ldots, 11\}$ and let $\GG = \{\G_1 = \{1, \ldots, 5\},~\G_2 = \{4, \ldots, 8\},~\G_3 = \{7, \ldots, 11\}\}$ be the acyclic group structure structure with $3$ groups of equal cardinality. Its intersection graph is represented in Fig.~\ref{fig:example}. 
Consider the $2$-group sparse signal ${\bf x} = [0~0~1~1~1~0~1~1~1~0~0]^\top$, with minimal group-cover $\mathcal{M}({\bf x}) = \{\G_1, \G_3\}$. 

\tikzstyle{place}=[circle,draw=black,fill=white,thick, minimum size=6pt, inner sep=0pt]
\begin{figure}
\centering
\begin{tikzpicture}[-,>=stealth',shorten >=1pt,auto,node distance=1.5cm, semithick]
	\node[place] (n1) at (0,0)    [label=below:$\G_1$] {};
	\node[place] (n2) at (1.5,1) [label=above:$\G_2$] {};
	\node[place] (n3) at (3,0)    [label=below:$\G_3$] {};
	
\draw (n1) to node {$\{4,5\}$} (n2);
\draw (n2) to node {} (n1);
\draw (n2) to node {$\{7,8\}$} (n3);
\draw (n3) to node {} (n2);
\end{tikzpicture}
\caption{\label{fig:example} The intersection graph for the example in Section~\ref{sec:discrete_vs_convex}}
\end{figure}
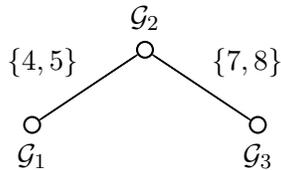

The dynamic program of Theorem \ref{theo:DP1}, with group budget $G = 2$, correctly identifies the groups $\G_1$ and $\G_3$. 
The TU linear program \eqref{eq:PR}, with $0 <  \lambda \leq 2$, also yields the correct group-cover.
Conversely, the decomposition obtained via \eqref{eq:atomic_norm} with unitary weights is unique, but is not group sparse. In fact, we have $\mathcal{S}({\bf x}) = \breve{\mathcal{S}}({\bf x}) = \GG$.
We can only obtain the correct group-cover if we use the weights $[1~d~1]$ with $d > \frac{2}{\sqrt{3}}$, that is knowing beforehand that $\G_2$ is irrelevant.

\begin{rem}
This is an example where the correct minimal group-cover exists, but cannot be directly found by the Latent Group Lasso approach. 
There may also be cases where the minimal group-cover is not unique. 
We leave to future work, to investigate which of these minimal covers are obtained by the proposed dynamic program and characterize the behavior of relaxations.
\end{rem}

\section{Generalizations}
\label{sec:generalizations}

In this section, we first present a generalization of the discrete approximation problem \eqref{eq:WMC} by introducing an additional overall sparsity constraint. 
Secondly, we show how this generalization encompasses approximation with hierarchical constraints that can be solved exactly via dynamic programming. 
Finally, we show that the generalized problem can be relaxed into a linear binary problem and that hierarchical constraints lead to totally unimodular matrices for which there exists efficient polynomial time solvers.

\subsection{Sparsity within groups}
In many applications, for example genome-wide association studies \cite{zhou2010association}, it is desirable to find approximations that are not only group-sparse, but also sparse in the usual sense (see \cite{simon2012sparse} for an extension of the group lasso). 
To this end, we generalize our original problem \eqref{eq:WMC} by introducing a sparsity constraint $K$ and allowing to individually select variables within a group. 
The generalized integer problem then becomes
\begin{equation}
\label{eq:GWMC}
\max\limits_{\boldsymbol{\omega} \in \mathbb{B}^\ngroups,~{\bf y} \in \mathbb{B}^\dim} \left \{ \sum_{i=1}^N y_i x_i^2 : \A^\mathfrak{G} \boldsymbol{\omega} \geq {\bf y}, \sum_{i=1}^\dim y_i \leq K , \sum_{j=1}^\ngroups \omega_j \leq G \right \} \; .
\end{equation}

The problem described above is a generalization of the well-known Weighted Maximum Coverage (WMC) problem. The latter does not have a constraint on the number of indices chosen, so we can simulate it by setting $K=N$. WMC is also well-known to be NP-hard, so that our present problem is also NP-hard, but it turns out that it can be solved in polynomial time for the same group structures that allow to solve \eqref{eq:WMC}.

\begin{theorem}
\label{prop:DP2}
Given an acyclic groups structure $\mathfrak{G}$, there exists a dynamic programming algorithm that solves \eqref{eq:GWMC} with complexity $\bigO(M^2GK^2)$.
\end{theorem}
\begin{proof}
The dynamic program is described in Appendix \ref{sec:dp_lp} alongside the proof that it has a polynomial running time.
\end{proof}

\subsection{Hierarchical constraints}
The generalized model allows to deal with hierarchical structures, such as regular trees, frequently encountered in image processing (e.g. denoising using wavelet trees).
In such cases, we often require to find $K$-sparse approximations such that the selected variables form a rooted connected subtree of the original tree, see Fig.~\ref{fig:hier}. 
Given a tree $\mathcal{T}$, the rooted-connected approximation can be cast as the solution of the following discrete problem
\begin{equation}
\label{eq:hier}
\max_{{\bf y} \in \BB^\dim} \left \{ \sum_{i=1}^N y_ix_i^2: \supp({\bf y}) \in \mathcal{T}_K \right \} \; ,
\end{equation}
where $\mathcal{T}_K$ denotes all rooted and connected subtrees of the given tree $\mathcal{T}$ with at most $K$ nodes.

\tikzstyle{notsel}=[circle,draw=black,fill=white,thick, minimum size=6pt, inner sep=0pt]
\tikzstyle{sel}=[circle,draw=black,fill=black,thick, minimum size=6pt, inner sep=0pt]
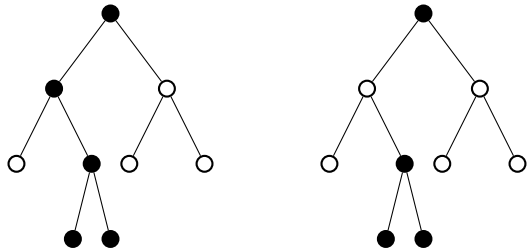
\begin{figure}
\centering
\begin{tabular}{cc}
\begin{tikzpicture}[level distance=10mm]
  \tikzstyle{level 1}=[sibling distance=15mm]
  \tikzstyle{level 2}=[sibling distance=10mm]
  \tikzstyle{level 3}=[sibling distance=5mm]
  \node[sel] {}
    child {node[sel] {} 
  		child{node[notsel] {}} 
		child{node[sel] {} 
			child{node[sel] {}} 
			child{node[sel] {}}}}
    child {node[notsel] {} 
		child{node[notsel] {}} 
		child{node[notsel] {}}};
\end{tikzpicture}
& \hskip 1cm
\begin{tikzpicture}[level distance=10mm]
  \tikzstyle{level 1}=[sibling distance=15mm]
  \tikzstyle{level 2}=[sibling distance=10mm]
  \tikzstyle{level 3}=[sibling distance=5mm]
  \node[sel] {}
    child {node[notsel] {} 
  		child{node[notsel] {}} 
		child{node[sel] {} 
			child{node[sel] {}} 
			child{node[sel] {}}}}
    child {node[notsel] {} 
		child{node[notsel] {}} 
		child{node[notsel] {}}};
\end{tikzpicture}
\end{tabular}	
\caption{\label{fig:hier} Hierarchical constraints. Each node represent a variable. (Left) A valid selection of nodes. (Right) An {\em invalid} selection of nodes.}
\end{figure}

This type of constraint can be represented by a group structure, where for each node in the tree we define a group consisting of that node and all its ancestors. 
When a group is selected, we require that all its elements are selected as well.
We impose an overall sparsity constraint $K$, while discarding the group constraint $G$.

For this particular problem, for which relaxed and greedy approximations have been proposed \cite{baraniuk1994signal, baraniuk1999optimal, jenatton2011proximal}, in Appendix \ref{sec:dp_hier}, we present a dynamic program that runs in polynomial time. 
\begin{theorem}
\label{th:dp_hier_gen}
The time complexity of our dynamic program on a general tree is $\mathcal{O}(NK^2D)$, where $D$ is the maximum number of children that a node in the tree can have.
\end{theorem}

While preparing the final version of this manuscript, \cite{cartis2013exact} independently proposed a similar dynamic program for tree projections on $D$-regular trees with time complexity $\mathcal{O}(NKD)$. 
Following their approach, we improved the time complexity of our algorithm to $\mathcal{O}(NKD)$ for $D$-regular trees. We also prove that its memory complexity is $\bigO (N \log_DK)$. 
A computational comparison of the two methods, both implemented in Matlab, is provided in Section \ref{sec:exp}, showing that our dynamic program can be up to $60\times$ faster, despite having similar {\em worst-case} time complexity.
\begin{prop}
\label{prop:dp_hier_reg}
The time complexity of our dynamic program on $D$-regular trees is $\mathcal{O}(NK^2D)$.
\end{prop}
\begin{prop}
\label{prop:dp_hier_reg_space}
The space complexity of our dynamic program on $D$-regular trees is $\bigO (N \log_DK)$.
\end{prop}

The description of the algorithm and the proof of its complexity, for both general and $D$-regular trees, can be found in Appendix \ref{sec:dp_hier}.

\subsection{Additional discrete relaxations}
By relaxing both the group budget and the sparsity budget in \eqref{eq:GWMC} into regularization terms, we obtain the following binary linear program
\begin{equation}
\label{eq:LP_hier}
(\boldsymbol{\omega}^\lambda, {\bf y}^\lambda) \in  \argmax\limits_{\boldsymbol{\omega} \in \BB^\ngroups, {\bf y} \in \BB^\dim} \left \{ \mathbf{w}^\top \mathbf{u} : \mathbf{u}^\top = [{\bf y}^\top~\boldsymbol{\omega}^\top~{\bf y}^\top],~\mathbf{Cu} \leq 0 \right \}
\end{equation}
where $\mathbf{w}^\top = [x_1^2,  \ldots, x_N^2,-\lambda_G\1_M^\top,-\lambda_K\1_N^\top]$ and $\mathbf{C} = [\Id_N,~-\A^\GG,~\mathbf{0}_N]$ and $\lambda_G, \lambda_K > 0$ are two regularization parameters that indirectly control the number of active groups and the number of selected elements.
\eqref{eq:LP_hier} can be solved in polynomial time if the constraint matrix $\mathbf{C}$ is totally unimodular. 
Due to its structure, by Proposition 2.1 in \cite{wolsey1999integer} and that concatenating a matrix of zeros to a TU matrix preserves total unimodularity, $\mathbf{C}$ is totally unimodular if and only if $\A^\GG$ is totally unimodular.
The next results proves that the constraint matrix of hierarchical group structures is totally unimodular.
\begin{prop}
Hierarchical group structures lead to totally unimodular constraints.
\end{prop}
\begin{proof}
We use the fact that a binary matrix is totally unimodular if there exists a permutation of its columns such that in each row the $1$s appear consecutively, which is a combination of Corollary 2.10 and Proposition 2.1 in \cite{wolsey1999integer}. 
For hierarchical group structures, such permutation is given by a depth-first ordering of the groups. 
In fact, a variable is included in the group that has it as the leaf and in all the groups that contain its descendants. 
Given a depth-first ordering of the groups, the groups that contain the descendants of a given node will be consecutive.
\end{proof}

The regularized hierarchical approximation problem, in particular
\begin{equation}
\label{eq:hier_reg}
\max_{{\bf y} \in \BB^\dim} \left \{ \sum_{i=1}^N y_ix_i^2 - \lambda\|{\bf y}\|_0,~\supp(\y) \in \mathcal{T}_K \right \} \; ,
\end{equation}
for $\lambda \geq 0$, has already been addressed by Donoho \cite{donoho1997cart} as the ``complexity penalized residual sum-of-squares'' and linked to the CART \cite{breiman1984classification} algorithm, which can found a solution in $\bigO(N)$ time. The {\em condensing sort and select algorithm} (CSSA) \cite{baraniuk1994signal}, with complexity $\bigO(\dim \log \dim)$, solves the problem where the indicator variable $\y$ is relaxed to be continuous in $[0, 1]$ and constrains $\|\y\|_1$ to be smaller than a given threshold $\gamma$, yielding rooted connected approximations that might have more than $K$ elements.

\section{\label{sec:exp}Pareto Frontier Examples}
The purpose of these numerical simulations is to illustrate the limitations of relaxations and of greedy approaches for correctly estimating the $G$-group cover of an approximation.

\subsection{Acyclic constraints}

We consider the problem of finding a $G$-group sparse approximation of the wavelet coefficients of a given image, in our case a view of the Earth from space, see left inset in Fig.~\ref{fig:exp2_1}.
We consider a group structure defined over the 2D wavelet tree. 
The wavelet coefficients of a 2D image can naturally be organized on three regular quad-trees, corresponding to a multi-scale analysis with wavelets oriented vertically, horizontally and diagonally respectively \cite{mallat1999wavelet}. 
We define groups consisting of a node and its four children, therefore each group has $5$ elements, apart from the topmost group that contains the scaled DC term and the first nodes of each of the three quad-trees.
These groups overlap only pairwisely and their intersection graph is a tree itself, therefore leading to a totally unimodular constraint matrix.
An example is given in the right inset in Fig.~\ref{fig:exp2_2}.
For computational reasons, we resize the image to $16 \times 16$ pixels and compute its Daubechies-4 wavelet coefficients. 
At this size, there are $64$ groups, but actually $52$ are sufficient to cover all the variables, since it is possible to discard the penultimate layer of groups while still covering the entire ground set.

Figures \ref{fig:exp2_1} and \ref{fig:exp2_2} show the Pareto frontier of the approximation error $\|{\bf x} - \hat{\bf x}\|_2^2$ with respect to the group sparsity $G$ for the proposed dynamic program. 
We also report the approximation error for the solutions obtained via the totally unimodular linear relaxation (TU-relax) \eqref{eq:PR_std} and the latent group lasso formulation (Latent GL) \eqref{eq:latent_gl} with $p = 2$, which we solved with the method proposed in \cite{mosci2010primaldual}. 
Fig.~\ref{fig:exp2_2} shows the performance of StructOMP \cite{huang2011learning} using the same group structure and of the greedy algorithm for solving the corresponding weighted maximum coverage problem.

We observe that there are points in the Pareto frontier of the dynamic program, for $G = 5, 10, 30, 31, 50$, that are not achievable by the TU relaxation, since they do not belong to its convex hull. Furthermore, the latent group lasso approach often does not yield the optimal selection of groups, leading to a greater approximation error for the same number of active groups and it needs to select all $64$ groups in order to achieve zero approximation error.
It is interesting to notice that the greedy algorithm outperforms StructOMP (see inset of Fig.~\ref{fig:exp2_2}), but still does not achieve the optimal solutions of the dynamic program. Furthermore, StructOMP needs to select all $64$ groups for obtaining zero approximation error, while the greedy algorithm can do with one less, namely $63$.



\begin{figure}
\centering
\includegraphics[width=0.95\columnwidth]{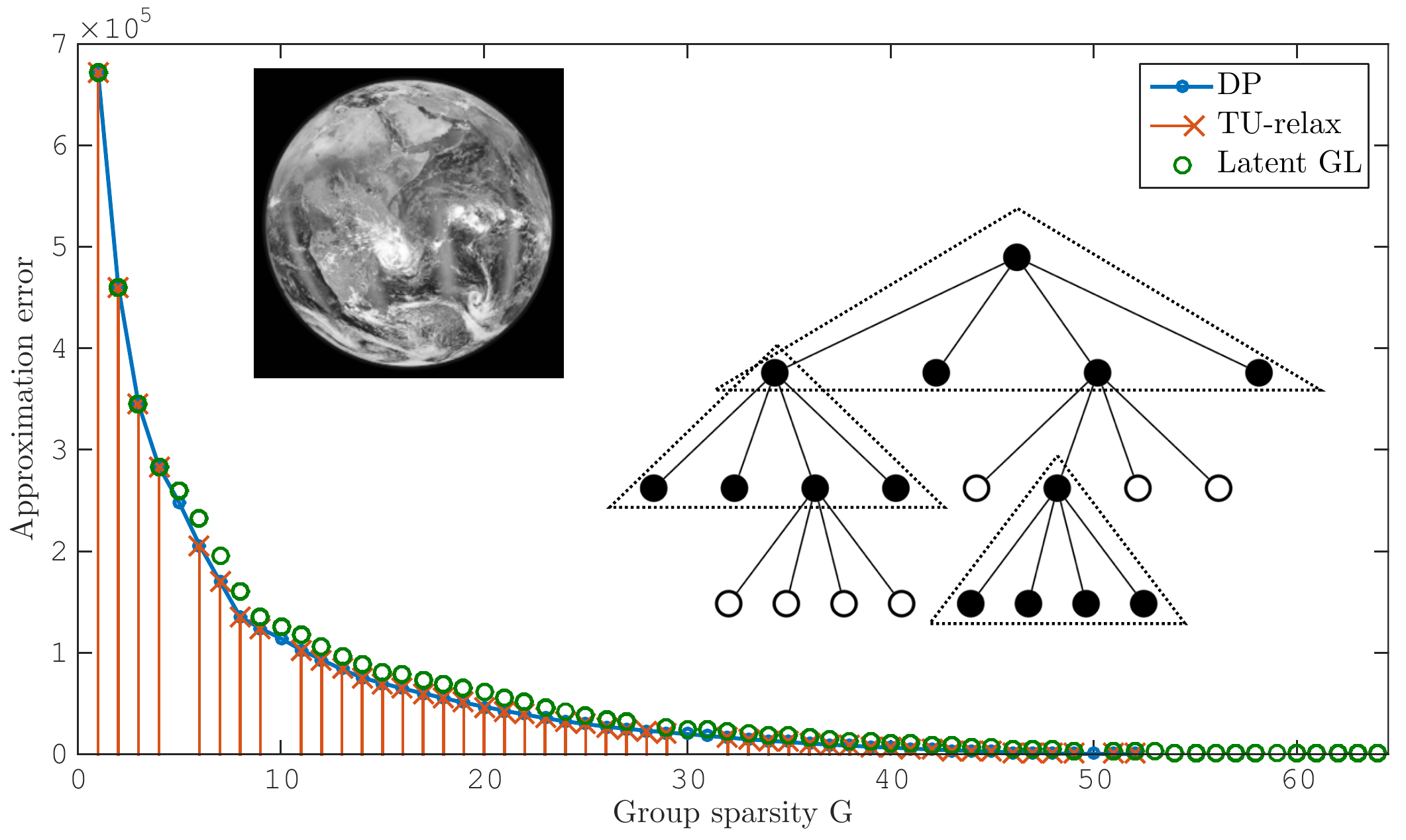}
\caption{\label{fig:exp2_1} Insets: (Left) Earth image used for the numerical simulation, after being resized to $16 \times 16$ pixels. (Right) Example of allowed support on one of the three wavelet quad-trees: The black circles represent selected variables and the empty circles unselected ones, while the dotted triangles stand for the active groups. Main plot: 2D Wavelet approximation on three quad-trees. The original signal is the wavelet decomposition of the $16 \times 16$ pixels Earth image. The blue line is the Pareto frontier of the dynamic program for all group budgets $G$. Note that, for $G \geq 52$, the approximation error for the dynamic program is zero, while the Latent Group Lasso approach needs to select all $64$ groups to yield a zero error approximation. The totally unimodular relaxation only yields the points in the Pareto frontier of the dynamic program that lie on its convex hull.}\end{figure}

\begin{figure}
\includegraphics[width=0.95\columnwidth]{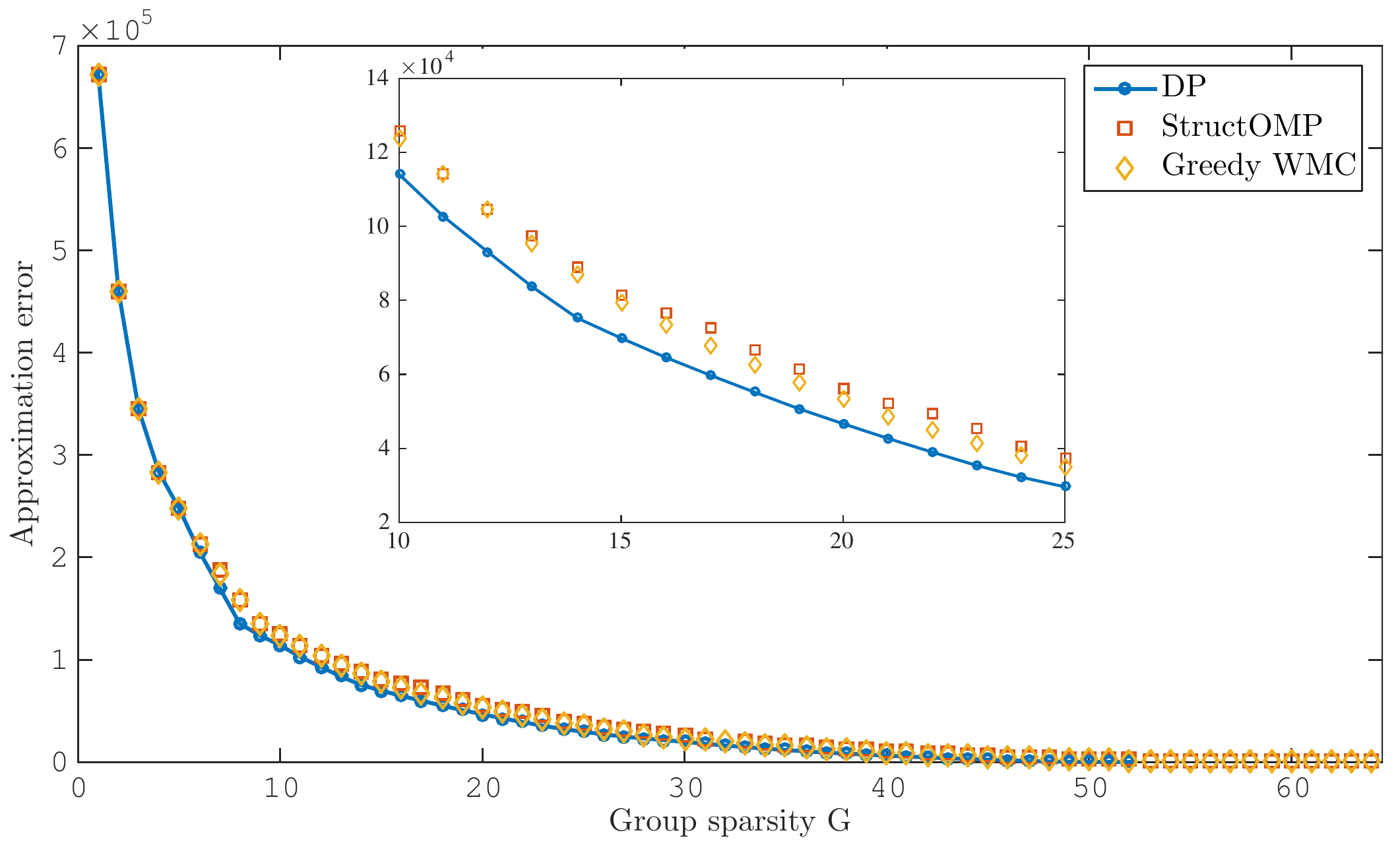}
\caption{\label{fig:exp2_2} 2D Wavelet approximation on three quad-trees. The original signal is the wavelet decomposition of the $16 \times 16$ pixels Earth image. The blue line is the Pareto frontier of the dynamic program for all group budgets $G$. 
Note that for $G \geq 52$ the approximation error for the dynamic program is zero, while StructOmp needs to select all $64$ groups to yield a zero error approximation. The greedy algorithm for solving the corresponding weighted maximum coverage problem obtains better solutions than StructOmp, but it still requires $63$ groups to yield a zero-error approximation.}
\end{figure}


\subsection{Hierarchical constraints}
We now consider the problem of finding a $K$-sparse approximation of a signal imposing hierarchical constraints.
We generate a piecewise constant signal of length $N = 64$, to which we apply the Haar wavelet transformation, yielding a $25$-sparse vector of coefficients ${\bf x}$ that satisfies hierarchical constraints on a binary tree of depth $5$, see Fig.~\ref{fig:exp}(Left).

We compare the proposed dynamic program (DP) to the regularized totally unimodular linear program approach, two convex relaxations that use group-based norms and the StructOMP greedy approach \cite{huang2011learning}.
The first convex relaxation \cite{rao2012signal} uses the Latent Group Lasso norm \eqref{eq:atomic_norm} with $p = 2$ as a penalty and with groups defined as all parent-child pairs in the tree. We call this approach Parent-Child.
This formulation will not enforce all hierarchical constraints to be satisfied, but will only `favor' them. 
Therefore, we also report the number of hierarchical constraint violations.
The second convex relaxation \cite{jenatton2011proximal} considers a hierarchy of groups where $\G_j$ contains node $j$ and all its descendants. 
Hierarchical constraints are enforced by the group lasso penalty $\Omega_{GL}({\bf x}) = \sum_{\G \in \GG} \|{\bf x}_\G\|_p$, 
where ${\bf x}_\G$ is the restriction of ${\bf x}$ to $\G$, and we assess $p = 2$ and $p = \infty$. We call this method Hierarchical Group Lasso. As shown in \cite{elhalabi2015totally}, solving $\min_\x \|\y - \x\|_2^2 + \lambda \Omega_{GL}(\x)$, for $p = \infty$, is actually equivalent to solving the totally unimodular relaxation with the same regularization parameter.
Once we determine the support of the solution, we assign to the components in the support the values of the corresponding components of the original signal. 
Finally, for the StructOMP\footnote{We used the code provided at \url{http://ranger.uta.edu/~huang/R_StructuredSparsity.htm}} method, we define a block for each node in the tree. The block contains that node and all its ancestors up to the root. By finely varying the regularization parameters for these methods, we obtain solutions with different levels of sparsity.

In Figures ~\ref{fig:exp}(Right), we show the approximation error $\|{\bf x} - \hat{\bf x}\|_2^2$ as a function of the solution sparsity $K$ for the methods.
The values of the DP solutions form the discrete Pareto frontier of the optimization problem controlled by the parameter $K$. Note that there are points in the Pareto frontier that do not lie on its convex hull, hence these solutions are not achievable by the TU linear relaxation.
As expected, the Hierarchical Group Lasso\footnote{We used the code provided at \url{http://spams-devel.gforge.inria.fr/.}} with $p=\infty$ obtains the same solutions as the TU linear relaxation, while with $p = 2$ it also misses the solutions for $K = 21$ and $K=23$.
The Parent-Child\footnote{We used the algorithm proposed in \cite{mosci2010primaldual}.}
approach achieves more levels of sparsity (but still missing the solutions for $K=2, 13$ and $15$), although at the price of violating some of the hierarchical constraints, i.e., we count one violation when one node is selected but not its parent.
The StructOMP approach yields only few of the solutions on the Pareto frontier, but without violating any constraints.
These observations lead us to conclude that, in some cases, relaxations of the original discrete problem or other greedy approaches might not be able to find the correct group-based interpretation of a signal.

In Fig.~\ref{fig:exp1_times}, we report a computational comparison between our dynamic program and the one independently proposed by Cartis and Thompson \cite{cartis2013exact}. We consider the problem of finding the $K = 200$ sparse rooted connected tree approximation on a binary tree of a signal of length $2^L$, with $L = 9, \ldots, 18$, whose components are randomly and uniformly drawn from $[0, 1]$. Despite the two algorithms have the same computational complexity, $\bigO(NKD)$ and are both implemented in Matlab, our dynamic program is between $20$ and $60$ times faster.


\begin{figure}
\centering
\begin{tabular}{cc}
\includegraphics[width=0.47\textwidth]{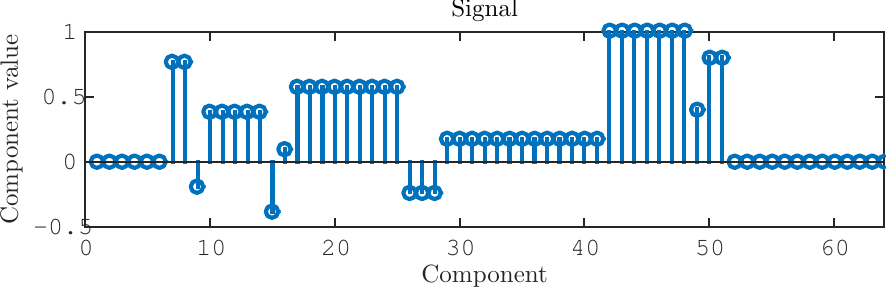} & \hspace{-0.3cm} \includegraphics[width=0.47\textwidth]{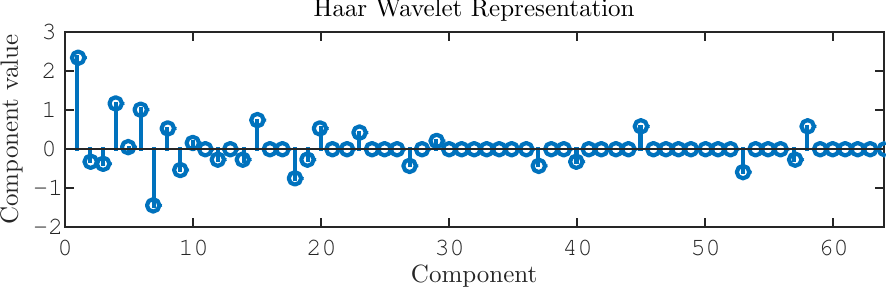} \\
\includegraphics[width=0.47\textwidth]{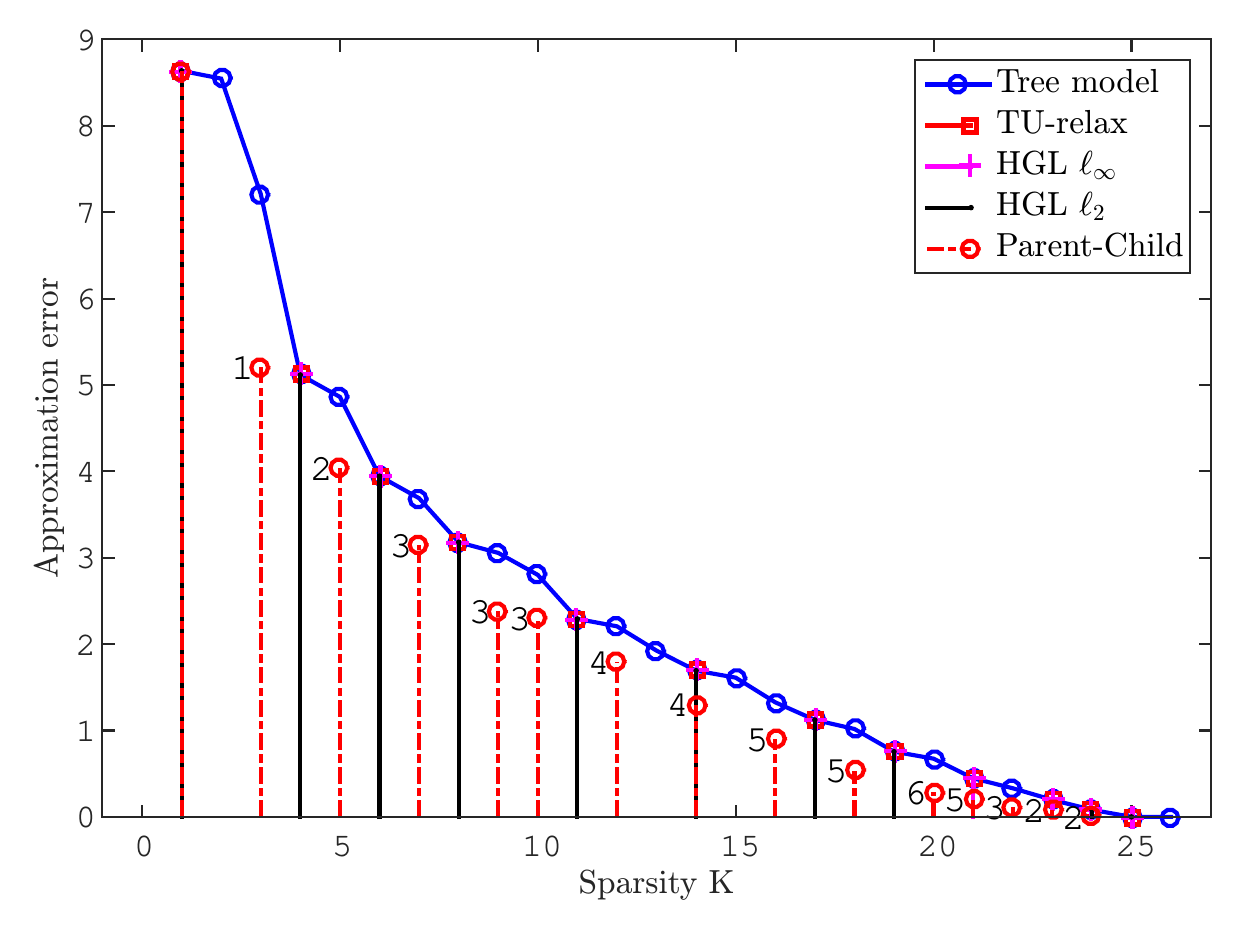} & \hspace{-0.3cm} \includegraphics[width=0.47\textwidth]{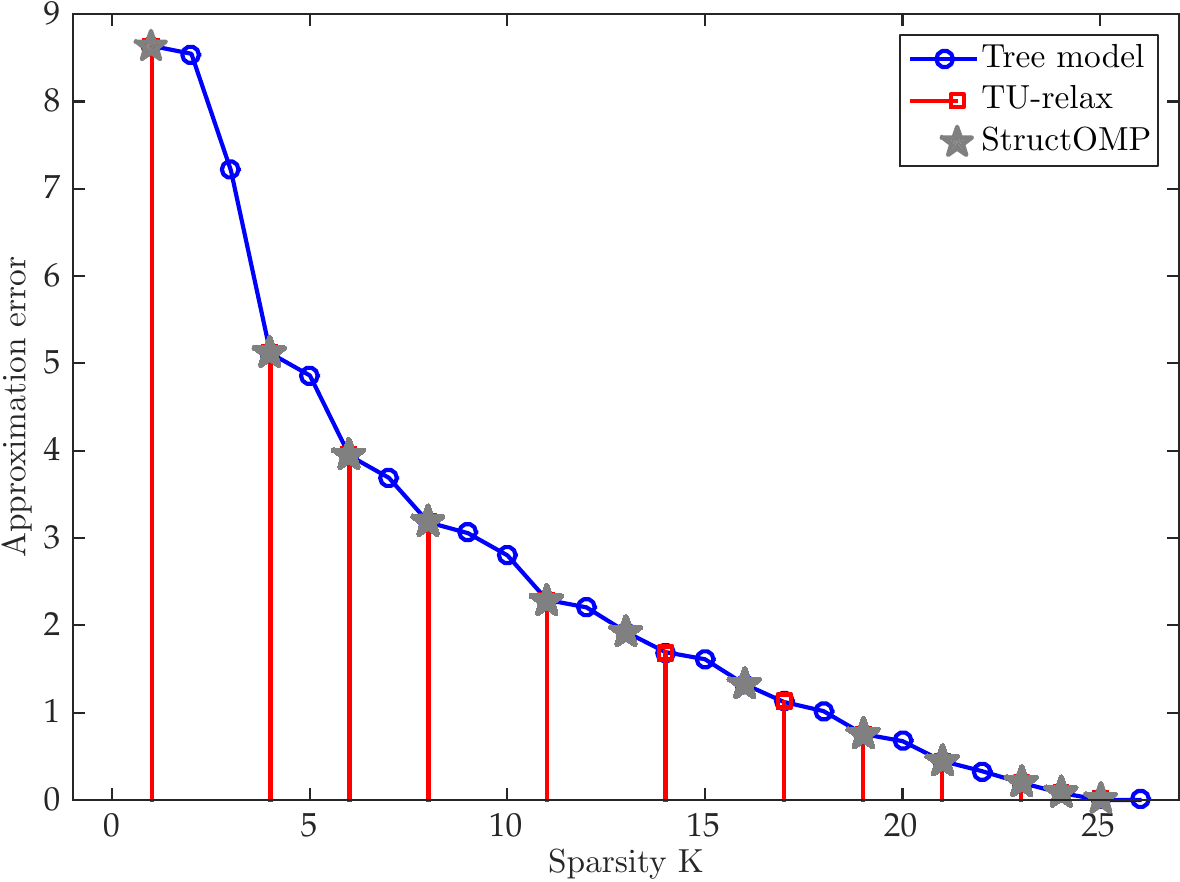}
\end{tabular}
\caption{\label{fig:exp} (Top) Original piecewise constant signal and its Haar wavelet representation. (Bottom) Signal approximation on the binary tree. The original signal is $25$-sparse and satisfies hierarchical constraints. The numbers next to the Parent-Child solutions indicate the number of hierarchial constraint violations, i.e., a node is selected but not its parent.}
\end{figure}

\begin{figure}
\noindent\begin{minipage}[b]{0.45\textwidth}
\centering
\footnotesize{
    \begin{tabular}{|c|c|c|c|}
    \hline
    N & DP & \cite{cartis2013exact} & Speed-up\\
    \hline
    $2^9$ 			& $0.007$ 	& $0.14$ 	& $20\times$\\
    $2^{10}$	 	& $0.012$ 	& $0.29$ 	& $23\times$\\
    $2^{11}$ 		& $0.025$ 	& $0.62$ 	& $25\times$\\
    $2^{12}$	 	& $0.048$ 	& $1.21$ 	& $25\times$\\
    $2^{13}$	 	& $0.093$ 	& $2.55$ 	& $27\times$\\
    $2^{14}$	 	& $0.19$ 		& $5.35$ 	& $29\times$\\
    $2^{15}$	 	& $0.37$ 		& $11.8$ 	& $32\times$\\
    $2^{16}$	 	& $0.76$ 		& $26.4$ 	& $35\times$\\
    $2^{17}$	 	& $1.54$ 		& $66.5$ 	& $43\times$\\
    $2^{18}$	 	& $3.14$ 		& $196$ 	& $62\times$\\
    \hline
    \end{tabular}}
    \label{table:exp1_times}
\end{minipage}
\hspace{1cm}
\begin{minipage}[c]{0.45\textwidth}
        \includegraphics[width=0.7\textwidth]{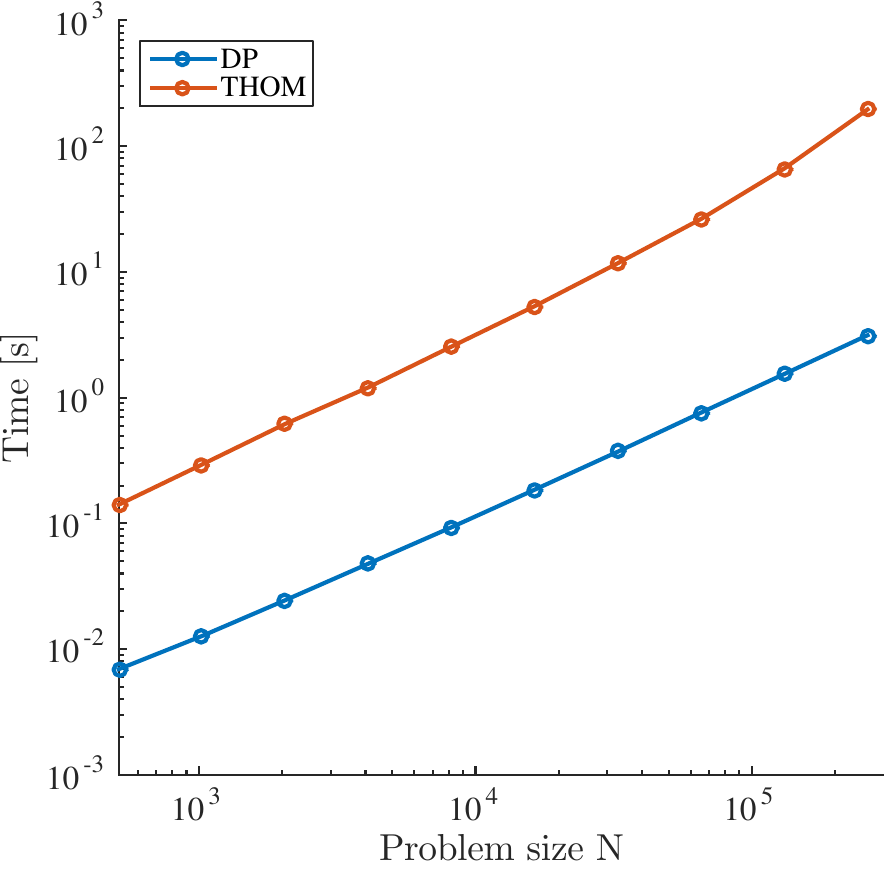}
\end{minipage}
\caption{\label{fig:exp1_times} Running times in seconds of the proposed dynamic program for hierarchical constraints and the one proposed by Cartis and Thompson \cite{cartis2013exact}. The sparsity budget is kept constant to $K = 200$ for all problem sizes.}
\end{figure}

\section{Conclusions}
\label{sec:conclusions}
Many applications benefit from group sparse representations. Unfortunately, our main result in this paper shows that finding a group-based interpretation of a signal is an integer optimization problem, which is in general NP-hard. To this end, we characterize group structures for which a dynamical programming algorithm can find a solution in polynomial time and also delineate discrete relaxations for special structures (i.e., totally unimodular constraints) that can obtain correct solutions. 

Our examples and numerical simulations show the deficiencies of relaxations, both convex and discrete, and of greedy approaches. 
We observe that relaxations only recover group-covers that lie in the convex hull of the Pareto frontier determined by the solutions of the original integer problem for different values of the group budget $G$ (and sparsity budget $K$ for the generalized model). 
This, in turn, implies that convex and non-convex relaxations might miss some important groups or include spurious ones in the group-sparse model selection. We summarize our findings in Fig.~\ref{fig:diagram}.

There remain several interesting open questions which beg for answers. 
Firstly, there still lacks an intuitive understanding of under which circumstances the relaxations are able to yield the correct solutions. 
Secondly, our analysis implicitly assumes an orthogonal basis for the description of signals. 
In many machine learning and compressive sensing applications however, the structures in signals emerge only after representing them onto an overcomplete basis, e.g. shearlets or sparse coding techniques. 
Therefore, it would be interesting to explore to which extent our results can be generalized to the overcomplete setting.

\begin{figure}
\centering
\includegraphics[width=0.6\textwidth]{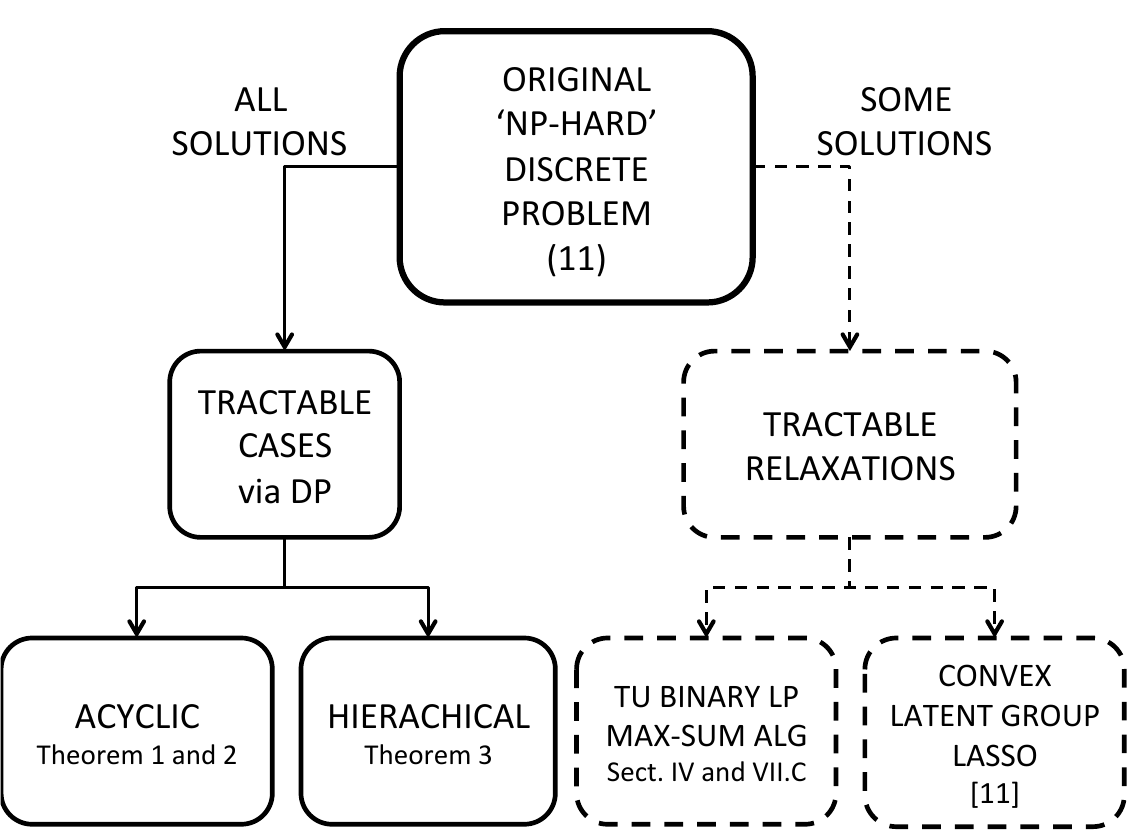}
\caption{\label{fig:diagram} Characterization of tractability for group-based interpretations.}
\end{figure}

\appendices

\section{Dynamical programming for solving \eqref{eq:GWMC} for loopless pairwise overlapping groups}
\label{sec:dp_lp}
Here, we give the proof of Theorem \ref{prop:DP2}. 
The proof of Theorem \ref{theo:DP1} follows along similar lines. 
We start by giving an intuitive understanding of the algorithm, followed by a formal description and proofs of correctness and complexity, both in time and space.

Problem \eqref{eq:GWMC} can be equivalently described by the following problem:

\paragraph*{Sparse Group Selection Problem (SGSP)} Given a signal ${\bf x} \in \Real^N$ and a group structure $\GG$ consisting of $\ngroups$ groups defined over the index set $\mathcal{N} = \{1, \ldots, N\}$, with each index having an associated (non-negative) weight (e.g., $x_i^2,~\forall~i \in \mathcal{N}$), find the optimal selection of at most $K$ indices, to maximize the sum of their weights, such that the indices are contained in a union of at most $G$ groups. In this paper, we frequently use the term {\em elements} in place of indices, and use the term {\em weight of i$^{th}$ element} to refer to the i$^{th}$ entry of the weights vector.\footnote{Note that since each element is non-negative, we can assume that the optimal solution will contain the maximum allowed $G$ groups, as well as $K$ elements, except in trivial cases. We will therefore often assume that the optimal solution has exactly $G$ groups and exactly $K$ elements. However, no generality is lost in our theorems by removing this assumption.}

In this form, the problem described above is a generalization of the well-known Weighted Maximum Coverage (WMC) problem, which is NP-hard. 
In fact WMC is just a special case of SGSP with $K = N$.
Although this makes it intractable in general, we show that this problem has some interesting structure. 
This structure allows us to build a dynamic program which can obtain the exact solution in polynomial time, for certain special classes of groups. 
We believe that this algorithm may be of independent interest outside the information theory community.

\subsection{An intuitive take on the Dynamic Programming Approach} 

\subsubsection{Failure of Na\"\i ve DP}

We first present an informal account of the ideas behind our method. 
The basic idea we use is dynamic programming, i.e., we build the solution to the global optimization problem from solutions to subproblems. 
In our case, the subproblems correspond to looking at a subset of the group structure, and solving the optimization problem for this case. 
What we could do is to start from a single group and keep adding more groups one at a time, updating the optimal solution at each step. 
One may na\"\i vely hope that such an approach would lead to the global solution. 
Unfortunately, this basic intuitive approach fails, as we illustrate next through some examples.

\begin{itemize}
\item[(a)] {\bf Example-1:} Consider the case of $N = 5$, with the weights being the vector $[5, 5, 2, 10, 6]$. For the sake of illustration, let the group structure be $\GG = \{\G_1, \G_2\}$, where $\G_1 = \{1,2\}, \G_2= \{2,3,4\}$. We wish to find the optimal solutions for the cases: 
\begin{enumerate}
	\item $K = 2, G = 1$; and 
	\item $K=3, G=2$. 
\end{enumerate}

The optimal solutions can be found simply by observation.
\begin{enumerate}
\item The optimal solution for $K=2, G=1$, has weight $15$, and involves selecting group $\G_2$, and elements $\{2,4\}$. 
\item The optimal solution for $K=3, G=2$, has weight $20$, and involves selecting both groups and elements $\{1,2,4\}$. 
\end{enumerate} 

\item[(b)] {\bf Example-2:} Consider the case of $N=5$, with the weights being the vector $[5, 5, 2, 10, 6]$. Let the set of groups be $\GG = \{\G_1, \G_2, \G_3\}$, where $\G_1 = \{1, 2\}$, $\G_2 = \{2,3,4\}$, $\G_3 = \{4,5\}$. We wish to find the optimal solution for the case: $K=4, G=2$. 
Once again, we can see that the optimal solution involves selecting groups $\G_1$ and $\G_3$, and elements $\{1,2,4,5\}$, for a total value of $26$.  
\end{itemize}

In the examples described above, the set of elements and their weights are the same. 
However, in the first example, any optimal solution for any meaningful values of the parameters G and K, involves $\G_2$. 
Yet, in the second example, we have a situation where the optimal selection does not involve $\G_2$, see Figure~\ref{fig:dp_intuition}



\begin{figure}
\centering
\includegraphics[width=0.6\textwidth]{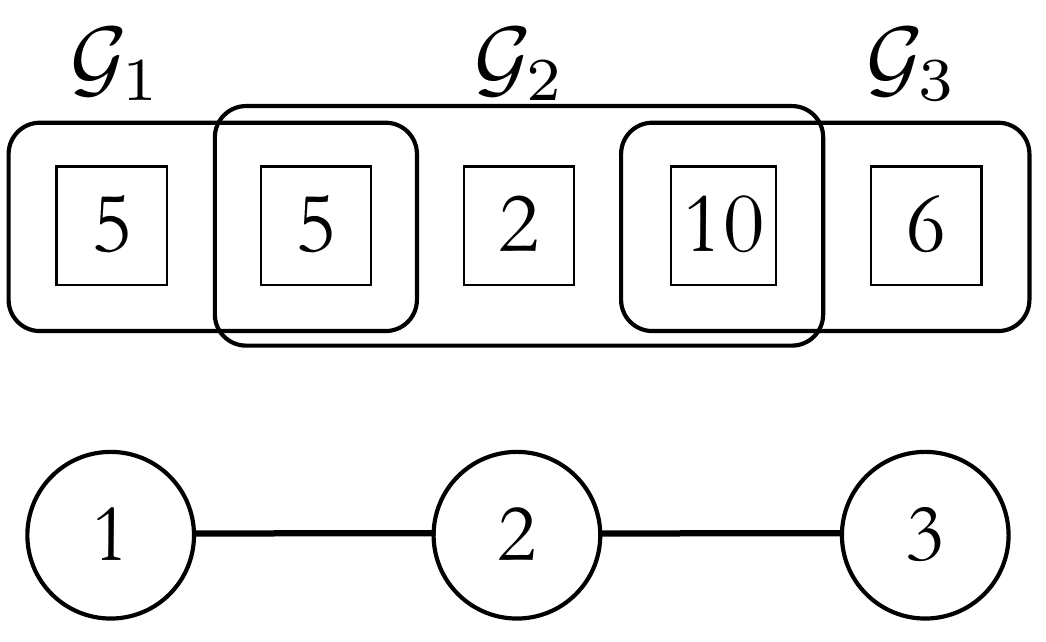}
\caption{\label{fig:dp_intuition} Failure of Na\"\i ve DP example: (top) Group structure. (bottom) Intersection graph. When we have only seen groups $\G_1$ and $\G_2$, the optimal solution to every subproblem involves choosing $\G_2$. After we explore $\G_3$, the optimal solution for $G=2, K=4$ no longer involves selecting $\G_2$.}
\end{figure}

\subsubsection{Boundary-cognizant DP}

As we illustrated above, the simple DP approach does not work.
When look at a subset of groups, some of which overlap with as yet unexplored groups, decisions regarding the overlapping groups are difficult to make. 
The reason is that the quality of a group in the view of the algorithm may decrease, if the high-weight elements in the group also happen to be contained in another overlapping group, which is seen in the future. 
While building partial solutions, we then need to consider both possibilities - an overlapping group is either included or excluded from the putative solution.
We now introduce some notation which allows us to describe these ideas more concretely.

Our algorithm is heavily based on the intersection graph of the group structure. 
Thus, we will frequently refer to the groups as `nodes' in our algorithm. 
Our approach involves exploring the nodes of the intersection graph one at a time and storing a list of optimal values from the explored nodes. 
These optimal values constitute the optimal weight of a $g$-group, $k$-element selection from the explored groups, for all $1 \leq g \leq G$ and for all $1 \leq k \leq K$. 
Further, we need to store these optimal values for each possible selection of the overlapping groups, so that we do not make decisions concerning such groups at the current step. 
In terms of the intersection graph, these overlapping groups are simply those nodes which belong to our currently explored set, but are also adjacent to some node which is not in the explored set. 
We call such nodes {\em boundary nodes}. 
Since our algorithm explores the intersection graph keeping track of all possibilities at the boundary nodes, it may fittingly be called a {\em boundary-cognizant Dynamic Program}.

Although this trick of being boundary-cognizant helps us get the correct solution, it can be expensive. 
Suppose we have $b$ boundary nodes at a certain step of the algorithm. 
Then the table of optimal values we seek to store has size $G K 2^{b}$, which is exponential in $b$. 
For an arbitrary intersection graph, this factor can indeed be exponential; for example, a complete graph with $M$ nodes will always have $M-1$ boundary nodes at the penultimate step. 
However, if we restrict the intersection graph to be a tree, then it turns out there is a way to explore the graph such that the number of boundary nodes in a graph with $M$ nodes is only $O(\log M)$. 
This property allows our algorithm to run in polynomial time on such graphs.

\subsection{Optimal substructure}
\label{sec:substructure}

We expose the optimal substructure of this problem below by highlighting two key properties: \textit{Groups-elements dichotomy} property and \textit{independence given the boundary} property. 
These provide sufficient evidence that an optimal solution to our problem can be efficiently constructed from optimal solutions to subproblems, indicating the correctness of the dynamic programming approach. Further, we will use a slight generalization of property-2 in the proof of correctness of our algorithm. 

\begin{enumerate}
\item {\em Groups-elements dichotomy:} 
Suppose we had access to an oracle who told us the set of $G$ groups that comprise the optimal solution to SGSP. Then we can easily recover the full solution using this information, by picking the $K$ largest-weight elements contained in the union of these $G$ groups.

Interestingly, the converse of the above is not true. If the oracle told us the list of $K$ elements contained in the optimal selection, but not the groups, the problem remains hard. Finding the $G$ groups that comprise the optimal solution is equivalent to finding a $G$-group cover for these $K$ elements, given that such a cover exists. If we could solve this task in polynomial time, the same algorithm would also solve the NP-Hard {\em Set Cover} problem in polynomial time.\footnote{The idea described here is not a formal reduction. It is possible that the additional structure possessed by the optimal solution would allow us to recover the groups in polynomial time. However, there seems to be no clear way to use this additional structure, so the only obvious way to recover the groups is to solve a set-cover problem, which is NP-hard.}

In a certain sense, the above shows that the difficult part of finding the optimal solution is selecting the groups. However, this does not imply that the element sparsity constraint is insignificant. It is easy to create problem instances where even a small change in $K$ significantly changes the optimal selection. \\

\item {\em Independence given the boundary:} 
Let $\GG$ be the complete set of groups, and let $\S \subset \GG$ be a subset of these groups. Let $\Bset(\S)$ be the boundary nodes of $\S$, that is the nodes in $\S$ that are connected to nodes in its complement, $\S^c$.
Once again, we assume the existence of an oracle who knows the true solution. Suppose this oracle tells us the following information: 
\begin{enumerate}
    \item The number of groups in $\S$ which are included in the optimal solution. Call this quantity $G_1$.
    \item The number of elements in the optimal solution, which occur in any of the groups in $\S$. Call this quantity $K_1$. 
    \item The boundary nodes included in the optimal solution. 
\end{enumerate} 
Then this information allows us to recover the optimal solution, by solving two independent optimization problems on the sets $\S$ and $\S^c$ respectively. \\

For ease of explanation, we refer to the set of boundary nodes included in the optimal selection as the set of `active boundary nodes', $\Bset_A(\S)$. Note that $\Bset_A(\S)$ is known as it is given to us by the oracle. Further, we call the set of elements included in $\Bset_A(\S)$ the set of `active boundary elements', or $\E_A$. \\

{\bf Recovery Method:} 
In order to recover the global optimal solution, we need to recover the selection of groups and elements in $\S$ and $\S^c$ respectively. 

We first describe the procedure for $\S$. 
Consider all possible ways of choosing $K_1$ elements contained in $G_1$ groups from $\S$, such that the set of chosen groups in $\Bset(\S)$ exactly matches $\Bset_A(S)$. 
Among these choices, the choice which has the maximum total weight of chosen elements gives us the selections of groups and elements in $\S$. 

Now we describe the procedure for $\S^c$. We know that the total number of selected groups in the set $\S^c$ equals $G_2 \triangleq G - G_1$. 
Similarly, we know that the total number of selected elements, from elements contained \emph{only} in $\S^c$ equals $K_2 \triangleq K - K_1$. 
We perform a `cleaning' operation on groups in $\S^c$, where we remove elements in $\E_A$ from these groups. 
Let the new set of groups thus obtained be called $\tilde{\S^c}$ (note that $\tilde{S^c}$ is in general not a subset of $\GG$). 
Then, we can recover the optimal selection of groups and elements in $\S^c$, by finding the maximum-weight $G_2$-group, $K_2$-element selection in $\tilde{\S^c}$. \\

{\bf Proof:}  The proof of these two statements is straightforward. 
First, we formally show how to break the true optimal solution into two disjoint components. 
After this, we argue that the two components constitute optimal solutions to smaller optimization problems. \\

Let us denote the set of groups and elements in the global optimal solution by $\GG^*$ and ${\E}^*$, respectively. 
We create two new group-element selections, roughly corresponding to $\S$ and $\S^c$, which we shall denote by $({\GG}_1,\E_1)$ and $({\GG}_2,\E_2)$ respectively. 
These two components are constructed as follows: 
\begin{enumerate}
\item The set of selected groups in $\S$ and $\S^c$ are already disjoint, so these are directly assigned to ${\GG_1}$ and ${\GG_2}$ respectively. 
\item For any element in ${\cal E}$ which occurs only in (groups in) $\S$, assign it to ${\cal E}_1$. 
\item For any element in ${\cal E}$ which occurs only in $\S^c$, assign it to ${\cal E}_2$. 
\item For any element which occurs in $\S$ as well as $\S^c$ (and hence in $\Bset(\S)$), first try to assign it to ${\cal E}_1$. That is, check if this element is contained in $\GG_1$, and if so assign the element to ${\cal E}_1$. If not, we assign it to ${\cal E}_2$. 
\end{enumerate}
We can verify the following properties:  
\begin{enumerate}
\item[i.] $\GG_1$ and $\GG_2$ form a partition of $\GG^*$, and similarly $\E_1$ and $\E_2$ form a partition of $\E^*$. 
\item[ii.] $({\GG}_1,\E_1)$, $({\GG}_2,\E_2)$ represent valid group-element selections over the sets of groups $\S$ and $\S^c$ respectively (i.e. $\E_1$ is contained in the union of groups in $\GG_1$, and similarly $\E_2$ is contained in $\GG_2$.)
\item[iii.] $(\GG_2, \E_2)$ can also be thought of as a valid selection over the set $\tilde{\S^c}$. This is because our definition of the components assigns any element in the active boundary groups to $\E_1$ over $\E_2$. 
\item[iv.] $|\GG_1| = G_1$, $|\E_1| = K_1$, $|\GG_2| = G_2$, $|\E_2| = K_2$, where $G_1$, $K_1$, $G_2$, $K_2$ are defined as above.\\
\end{enumerate}

We are now ready to prove the correctness of the recovery method.

Let us first consider $\S^c$. 
Suppose that contrary to our claim above, $(\GG_2,\E_2)$ does not constitute an optimal solution for $\tilde{\S^c}$, i.e., there exists another $G_2$-group, $K_2$-element selection on $\tilde{\S^c}$, namely $(\GG_2', \E_2')$, such that the total weight of elements in $\E_2'$ is larger than that in $\E_2$. 
Then we could improve the optimal solution by considering the group-element selection $(\GG_1 \cup \GG_2', \E_1 \cup \E_2')$. 
Note that it is impossible for $\E_1$ and $\E_2'$ to select the same element twice, and hence the above represents a valid $G$-group, $K$-element selection over $\GG$. 
Also, since $weight(\E_2') > weight(\E_2)$, we have $weight(\E_1 \cup \E_2') > weight(\E_1 \cup \E_2) = weight(\E^*)$, so this is an improvement over the selection $({\GG}^*, {\E}^*)$. 
But this contradicts the optimality of the latter solution. 
Hence, our assumption must be false, i.e., ($\GG_2, \E_2$) comprises an optimal solution to the $G_2$-group, $K_2$-element selection problem for $\tilde{\S^c}$.\\

An identical argument shows that ($\GG_1, \E_1$) represents an optimal $G_1$-group, $K_1$-element selection over $\S$, among all group-element selections for which the set of chosen nodes from $\Bset(\S)$ equals exactly $\Bset_A(\S)$. This proves the correctness of our recovery method. $\qed$ \\




\end{enumerate} 


\subsection{Overview of our Algorithm}

%

Our algorithm explores the acyclic intersection graph one node at a time, storing the optimal solution among the visited nodes and eventually leading to the optimal solution for the entire graph.
It is described by two rules: the {\em Value Update Rule} and the {\em Graph Exploration Rule}. 
\begin{enumerate}
	\item \emph{Graph Exploration Rule}: This rule takes as input a given tree graph, and outputs an order of exploring the graph so as to minimize the number of encountered boundary nodes. 
	\item \emph{Value Update Rule}:  The Value Update Rule determines how to update the list of optimal values when we explore a new node. 
\end{enumerate}

We first describe the Value Update Rule. While doing so, we assume that the nodes of the graph have been labelled $1,2,\dots, M$ in some suitable manner, and explore them in this order. In order to lay the foundation for describing the update rule, we will first define the table of optimal values maintained by our algorithm, and ensure that the given data is in suitable format.

\subsection{Table of optimal values}


We describe the set of optimal solutions stored by our Table of optimal values. 
Abstractly, this table can be thought of as a mathematical function with $5$ different parameters. These are described below:

\begin{itemize}

\item {\bf Explored Set} : $\S \subseteq \GG$\\
This is any subset of nodes of the intersection graph. It represents the set of nodes currently visited by our algorithm. 

\item {\bf Group Count} : $g \in \{1, 2, \dots, G\}$. \\
This is the maximum number of groups we are allowed to select.

\item {\bf Element Count} : $k \in \{1, 2, \dots, K\}$. \\
This is the maximum number of elements we are allowed to select.

\item {\bf Boundary Set Vector} : $\b = (b_1, b_2, \dots, b_B),\   b_i \in S \ \forall i$, with $B = \dimv(\b)$. \\
This is any subset of the explored set $\S$, represented in vector form. 
We allow $\b$ to be an {\em empty vector}, which we denote by $\emptyset.$\footnote{We do not give a precise definition of empty vector in this text. Informally, it can be thought of as a vector of $0$ elements, very similar to an empty set.}

\item {\bf Boundary Indicator Vector} : $\I_\b \in \{0,1\}^B$. \\
This is a binary vector of size $B$. Given a boundary set vector, $\b$, for each $i \in \{1, \dots, B\}$, the $i$-th component of $\I_\b$ is either $0$ or $1$, representing whether the group $\b_i$ is selected or excluded in the optimal selection. 
We also allow $\I_\b$ to be an empty vector.

\end{itemize}

We now define our optimal values function as follows.

{$\bf F(\S, g, k, \b, \I_\b)$} represents the maximum weight obtainable by selecting at most $k$ elements contained in a union of at most $g$ groups from the set $\S$, with the choice of selections among the set of boundary nodes $\b$ given by $\I_\b$. This function is defined for the entire range of its arguments mentioned above.

Although the function is defined for all $\S \subseteq \GG$, in practice we explore the nodes one at a time, in serial order. 
Thus, we only need to keep track of $M$ different sets of explored nodes, 
where the $i$-th set, $\S_i$, consists of groups $\G_1,\G_2, \dots, \G_i$, for all $ i \in \{1, \dots, M\}$. 
Furthermore, we only see $M$ different sets of boundary nodes for a given intersection graph, $\Bset(\S_i)$.
In certain intermediate steps we shall find it convenient to use in place of $\Bset(\S_i)$, a different set than the actual set of boundary nodes.

If we fix $\S = \S_i$ and $\b = \b_i$\footnote{Technically $\b_i$ is a vector, and involves both a set of elements and an ordering over the elements. But this ordering is really a matter of notation; we will care only about the set of boundary nodes, and not the order, in our algorithm.} and vary other parameters over their respective ranges, we obtain the complete list of values stored by our algorithm at the $i$-th step. 
Note that the number of such stored values equals $G \cdot K \cdot 2^{B}$, with $B = \dimv(\b)$. \\


\subsection{Data Format and Notation} 

Without loss of generality, we can assume that each group has no more than $K$ elements. 
Further, we will assume that the indices in each group are specified in decreasing order of weights. 

In case the above assumptions are not met a-priori, we can do some preprocessing on the given data. 
Since we know that each group consists of at most $N$ elements, we can pick the largest $K$ elements and then sort them in $O(N + K \log N)$ time.\footnote{ This can be done by building a max-heap of all $N$ elements and then extracting the topmost element $K$ times.}
Since we need to do this for each one of the $M$ groups, this leads to a total complexity of $O(MN + MK \log N)$. 

While describing the complexity of our main algorithm, we will assume that the groups are already represented in the above canonical form. Hence, we will not consider the above term in our expression for time complexity.


Next, we formally define some notation that we use in our description of the value update rule. 

{\bf Concatenation Operator:} Given two vectors $\x$ and $\y$ of lengths $m$ and $n$ respectively, we define the vector `$x$ concatenated with $y$', written as $\x.\y$, to be an $m+n$-length vector which consists of entries of $\x$ followed by entries of $\y$. 



{\bf Best-k operator:} We define a function $H(\mathcal{S},k)$ to represent the optimal value for choosing $k$ elements from a set $\S$. 
The set $\S$ could be a single group, a union of groups, or any well-defined collection of elements. 
As noted earlier, $H(\S, k)$ simply equals the sum of the $k$ largest weight elements in $\S$.

\subsection{Value Update Rule}

We shall now describe the Value Update Rule. This rule shows us how to find the optimal solution to SGSP, which is represented by the value: $F(\GG , G, K, \emptyset, \emptyset)$.  \\

{\em Base Case.} We start with $\S_0 = {\emptyset}$. For this case, all values of $F$ are set to $0$: 
$F(\emptyset , g, k, \emptyset, \emptyset) = 0~\forall g, k$. \\


{\em Update.} 
The update case describes how to recompute the list of optimal values when we explore a new node. 
We shall apply this rule a total of $M$ times, exploring one new node from the graph each time, and updating our table of values. 
At the end, we can simply read off the solution from the appropriate entry of the table. 

Since we explore the nodes in serial order, at the $i$-th step, our explored set will consist of nodes $1,2, \dots, i$. 
As mentioned earlier, we denote our explored set after the $i$-th step as $\S_i$, and the boundary set vector at this time as $\b_i$. 
We use the notation $\G_{j}$ to refer to the $j$-th group, which is also the $j$-th node of the intersection graph as per our chosen ordering. 
At the end of the $i$-th step, we assume that we have stored the values of $F$ for the explored set $\S_i$ and boundary set vector $\b_i$ for each possible value of parameters $g$, $k$, and the indicator variable $\I_{\b_i}$, in their respective ranges. Thus, the following values are available to us: 
$$F(\S_i,g,k,\b_i, \I_{\b_i})\quad \text{for all} \quad 1 \leq g \leq G \quad \text{and}\quad 1 \leq k \leq K \quad \text{and all}\quad  \I_{\b_i} \in {\{0,1\}}^{B_i}, \text{ where }B_i = \dimv(\b_i).$$ 
Our objective is to extend these values to the case when we have explored the $i+1$-th node. In other words, defining $\S_{i+1} \triangleq \S_i \cup \{\G_{i+1}\}$, we wish to obtain the following set of values:  \\
$$
F(\S_{i+1},g,k,\b_{i+1}, \I_{\b_{i+1}}) \quad \text{for all} \quad 1 \leq g \leq G \quad \text{and} \quad 1 \leq k \leq K \quad \text{and all} \quad \I_{\b_{i+1}} \in {\{0,1\}}^{B_{i+1}}, \quad \quad \quad
$$ 
where $\b_{i+1}$ represents the boundary nodes at time $i+1$ in vector form. 

We now describe our method for obtaining these values. 
When we first consider node $i+1$, we treat it as a new boundary node and compute the optimal values for it being included or excluded from the putative solution. 
After this, we test for boundary nodes that have fallen into the interior of the explored set.  
For these redundant boundary nodes, we no longer need to store two separate values for the node being included or excluded, so we condense these into a single value. 
Our update rule thus consists of $3$ steps:

\begin{enumerate}
\item[1)] {\em The new node is excluded}. 

In this case, we are computing the optimal value for selecting $k$ elements contained in a union of $g$ groups among the first $(i+1)$ groups when the $(i+1)$-th group is not selected, and the groups in $\Bset_{i}$ are selected as per the indicator variables. 
Since the $(i+1)$-th group is not chosen, all our groups and elements must be chosen from among the first $i$ groups, with the same restrictions on the choice of boundary nodes. 
Hence, all optimal values for this case are equal to the corresponding values for $\S_i$. 
\begin{align*}
& F({\S}_{i+1}, g, k, \b_i . (\G_{i+1}), \I_{\b_i} . (0)) = F({\S}_{i}, g, k, \b_i, \I_{\b_i})
\end{align*}
for all $1 \leq g \leq G$ and $1 \leq k \leq K$ and all $\I_{\b_i} \in \{0,1\}^{B_i}$.\\

\item[2)] Case (a): {\em The new node is included and does not overlap with any explored node}.

In this case, we are computing the optimal values for the case when the $(i+1)$-th node is selected. 
Hence we can choose at most $g-1$ nodes from the first $i$ nodes. 
We first compute the sum of the optimal value for choosing the best $\ell$ elements from the new node and the optimal value for choosing $k - \ell$ elements from $g-1$ nodes in $S_i$, for any $\ell$ such that $1 \leq \ell \leq k$.
Then, the new optimal value for each $g$ and $k$ is given by taking the maximum of these sums over $\ell$.
To ensure that our optimal values are computed with selections of nodes in $B_i$ being specified by the indicator variables, we use the same values of indicators when computing the second term in the above sum.
\begin{align*}
& F({\S}_{i+1}, g, k, \b_i . (\G_{i+1}), \I_{\b_i} . (1)) \\
&  = \max_{1 \leq \ell \leq k} \left\{ F({\S}_{i}, {\bf g-1}, \boldsymbol{ k-\ell}, \b_i, \I_{\b_i}) + H({\G_{i+1}},\ell) \right\}
\end{align*}
for all $1 \leq g \leq G$ and $1 \leq k \leq K$ and all $\I_{\b_i} \in \{0,1\}^{B_i}$. \\

\item[2)] Case (b): {\em The new node is included but overlaps with some explored nodes}.

The update rule is the same as for case (a), but the elements in the region of overlap between the new node and the {\em selected} explored nodes must not be considered as being part of the new node. 
For this step, we need to know exactly which nodes have been chosen while computing an optimal value. 
This is the reason why we need to store separate values for each boundary node.



\begin{align*}
& F({\S}_{i+1}, g, k, \b_i . (\G_{i+1}), \I_{\b_i}. (1))\\
& = \max_{1 \leq \ell \leq k} \left\{ F({\S}_{i}, {\bf g-1}, \boldsymbol{ k-\ell},\b_i, \I_{\b_i}) + H(\mathcal{C}, \ell) \right \}
\end{align*}
for all $1 \leq g \leq G$ and $1 \leq k \leq K$ and all $I_{B_i} \in \{0,1\}^{B_i}$, where
$$
\mathcal{C} \triangleq \G_{i+1} \setminus \mathop{\bigcup}_{\substack{j \in \{1,\dots,B_i\} \\ \I_{\b_i}(j) = 1}} \b_i(j)
$$
That is we ``clean'' $\G_{i+1}$ of the overlap with the currently selected boundary nodes. \\


\item[3)] {\em Condensation.}

After performing the above steps, the number of stored values will be doubled. 
We can reduce them: for each boundary node which has fallen into the interior of the explored nodes, we combine the optimal values for it being selected or excluded, into a single value by taking the larger of the two values. 
Each such operation reduces the number of stored values by half and we perform it after each value update. 
Unlike the earlier steps, this step may have to be performed multiple times in a single update. 

Suppose $\b_i'$ is the current boundary set vector for which we have maintained optimal values. 
Suppose $\G_j$ is a node in $\b_i'$ which is not present in $\b_{i+1}$. For notational convenience, we will now assume that the group $\G_j$ has been moved to the end of the $\b_i'$ vector.  Define $\b_i''$ to be the vector of length $\dimv(\b_i') - 1$, consisting of all entries of $\b_i'$ except the last. Thus, we can write ${\b_i'} = {\b_i''} . \{{\G_j}\}$. Then we can reduce the boundary set vector from $\b_i'$ to $\b_i''$, as follows: 
\begin{align*}
F({\S}_{i+1}, g, k, \b_i'', \I_{\b_i''}) = \max \lbrace &F({\S}_{i+1}, g,k, \b_i', \I_{\b_i''} . (0)), F({\S}_{i+1}, g,k, \b_i', \I_{\b_i''} .  (1)) \rbrace
\end{align*}
for all $1 \leq g \leq G$ and $1 \leq k \leq K$ and for all $\I_{\b_i''} \in \{0,1\}^{B_i''}$, where $B_i'' = \dimv(\b_i'')$. \\

\end{enumerate} 

\subsubsection*{\bf Proof of correctness}


The correctness of our algorithm relies on the correctness of the value update rule. 
Below, we argue for the correctness of this rule for each of its $3$ steps.

\begin{itemize}

\item Step 1: The correctness of this step is self-evident. \\

\item Step 2, case (a): Since this step is a special case of step 2, case (b), it is sufficient to prove correctness of the latter. \\

\item Step 2, case (b): We prove the correctness of this step using the optimal substructure property 2 described in section ~\ref{sec:substructure}. 

Our task is to find the optimal selection of $g$-groups and $k$-elements from the set $\S_{i+1} \triangleq \S_i \cup \G_{i+1}$, when $\G_{i+1}$ is selected, and nodes in $\b_i$ are selected according to $\I_{\b_i}$. We now consider only the graph consisting of nodes in $\S_{i+1}$. With reference to the substructure property, choose the set $\S$ to be equal to $\S_i$. Critically, note that all groups in $\Bset(\S)$ are contained in $\b_i$, and thus we store optimal values separately for these. 

Although the substructure property 2 was derived on a graph with no additional information, it is equally well-applicable when certain groups (such as $\G_{i+1}$, and groups in $\b_i$) are constrained to be selected or excluded in the optimal solution. 
This property had three preconditions, one of which was the knowledge of boundary nodes in the optimal solution. 
This is trivially true, since in this particular optimization problem, the selection of groups in $\b_i$ is already fixed by $\I_{\b_i}$. 
Then, the property shows us that if we also know the number of groups and elements chosen from the two parts of the graph, we can recover the optimal solution over $\S_{i+1}$ by solving two separate optimization problems over $\S_i$ and $\G_{i+1}$ respectively.

Here, we know that exactly $g-1$ groups must be selected from $\S_i$, and (obviously) one group chosen from $\{\G_{i+1}\}$. 
However, we do not know the number of elements chosen from $\S_i$. 
Hence, we consider all possibilities by varying a parameter $\ell$ for the number of selected elements contained exclusively in $\G_{i+1}$, from $1$ up to $k$. 
More precisely, $\ell$ represents the number of elements chosen from $\mathcal{C}$, where $\mathcal{C}$ is the set of elements obtained by cleaning $\G_{i+1}$ of overlap with active boundary nodes in $\b_i$.  
This leads us to solve two independent optimization problems - find the best selection of $\ell$ elements from $\mathcal{C}$, and the best $(g-1)$-group, $(k - \ell)$-element selection from $\S_i$, respecting boundary node constraints.

Solving the optimization problem over $\mathcal{C}$ is trivial: simply choose the top $\ell$ elements. Solving the problem over $\S_i$ need not actually be carried out, since we have already stored all the relevant optimal solutions in the previous step. This value is stored in the $F$-function, in the entry $F(\S_i, g-1, k-\ell,\b_i, \I_{\b_i})$. Thus, by maximizing the sum of these optimal values and the best-$\ell$ selection in $\mathcal{C}$, over all $\ell$ from $1$ up to $k$, we obtain the optimal solutions for $\S_{i+1}$. \\

\item Step 3: This is the condensation step. The correctness of this step follows from the interpretation of the objective function - $F(\S, g, k, \b, \I_{\b})$ represents the optimal values for $g$-group $k$-element selections, when the choices of groups in $\b$ are fixed by $\I_{\b}$. Thus, for groups that are not in $\b$, we need to consider both whether the node is included or excluded. 
Therefore, in order to remove a node from the set $\b$, we simply take the maximum value of the two cases. \\

\end{itemize}

\subsubsection*{\bf Running Time}
The running time of our algorithm is determined by 2 steps - Value Update rule and the Graph Exploration algorithm. 
As we explain later, the exploration rule can be implemented independently and is computationally much faster, so the time complexity is determined by the value update rule. We analyze the complexity of each step of the update rule below. \\

{\em Complexity of step 1}:  
All optimal values for this case are simply the optimal values computed before the node is explored. 
Thus, the update in this case corresponds simply to a table-copying operation. 
In fact, this copying can be avoided entirely by some clever bookkeeping; all we need to do is remember where the appropriate values are stored in memory. 
Thus, this step is very inexpensive from a computational point of view. \\

{\em Complexity of step 2, case (a)}:  
Observe that the total number of values to be computed on the LHS of the update rule equals $G K 2^{B_i}$. 
To compute one such value, we need to take a maximum over $K$ different numbers on the RHS. We will show that each of these numbers can effectively be obtained in $O(1)$ time. 
Computing one of these numbers involves the sum of two terms. 
The first term is an optimal value that is already stored, so it merely involves a table lookup. 
The second term, $H({\G_{i+1}},\ell)$ involves taking the sum of $\ell$ largest numbers in the group $\G_{i+1}$. 
Since the elements in $\G_i$ are described to us in descending order of weights (by assumption), this is equivalent to finding the sum of the first $\ell$ elements. 
Since each successive sum differs from the previous sum in only one element, we can compute each sum by doing just one additional operation.  
Hence, computing the $K$ different numbers on the RHS takes only $\bigO(K)$ time. 
Combining this with the total number of values on the LHS, gives us an expression for complexity as $\bigO(G\, K^2 \, 2^{B_i})$. \\  

{\em Complexity of step 2, case (b):} 
The new operation that we need to perform here, compared to case (a), is the ``cleaning'' operation performed on the $i+1$-th node. 
This operation is independent of the parameters $g, k, \ell$, and depends only on the indicator variables $\I_{\b_i}$. 
Hence, we can perform our updates by first fixing $\I_{\b_i}$, and then varying $g$ and $k$. 
In this way we do the cleaning operation a total of $2^{B_i}$ times. 
The time required for the cleaning operation is equal to the time required to go through each of the $K$ elements in $\G_{i+1}$, and checking whether the element is also contained in any of the groups whose indicator variable is set to $1$. 
By doing some simple preprocessing (e.g. sorting indices in some canonical order), checking membership of an element in a group can be done in $\bigO(\log_2 K)$ time, by binary search. 
Thus, the time required for one cleaning operation is $\bigO(B_i K \log K)$. 
Hence the total time required for all cleaning operations in one step equals $\bigO(2^{B_i} B_i K \log K)$. 
Combining this with the expression obtained in step 2, case (a), the time complexity of this update step equals $\bigO(GK^2 2^{B_i} + K B_i 2^{B_i} \log K)$. \\

{\em Complexity of step 3:} 
Since condensation removes an explored node from the boundary set forever, it will have to be performed at most $M$ times in the entire algorithm. 
Since the set of boundary nodes at each step is fully determined by the intersection graph and the exploration ordering, these can be precomputed without significant time cost. 
Hence, we assume these are available to us and ignore their complexity. 
Then the complexity of a single condensation step is determined only by the number of values that need to be condensed, and is given by $\bigO(G K 2^{B_i'})$, which also equals $\bigO(G K 2^{B_i})$. \\

{\em Overall time complexity:}
Among the above, the most expensive case is step 2, case (b). 
The complexity of this step as obtained earlier equals $\bigO(GK^2 2^{B_i} + K B_i 2^{B_i} \log K)$, for the $i+1$-th value update. 
We need to perform this step $M$ times, with the parameter $i$ varying from $0$ to $M-1$ in the above expression.

Let $B^*$ be the maximum number of boundary nodes encountered by the algorithm at any step, i.e., $B^* = \max_i B_i$. 
Then the running time of our update algorithm is bounded by $\bigO(M(2^{B^*} K^2 G + 2^{B^*} {B^*} K \log K))$. 
Our graph exploration rule allows us to explore the graph so that $B^*$ is logarithmic in $M$, specifically $B^* \leq (\log_2 M +1)$. 
Hence $2^{B^*} = \bigO(M)$. 
Using this in our above expression, we see that the complexity becomes $\bigO(M^2 K^2 G + M^2 K \log M \log K)$  which shows that our algorithm is polynomial time. 
If we ignore logarithmic terms, we can write the complexity more compactly as $\bigO(M^2 K^2 G)$.

\subsubsection*{\bf Space Complexity and Backtracking} 
We now look at the amount of space (memory) required by our algorithm. 
To account for this, we also need to describe how we will backtrack, i.e., how we find the optimal selection of groups and elements. 
Note that the method described above yields the optimal value for selecting $K$ elements from $G$ groups, but does not immediately tell us which groups are selected. 
We chose a backtracking method which is time-efficient, but involves storing a fair amount of data. 
Specifically, we store the optimal values obtained at each step of the value update rule prior to condensation, i.e., $F(\S_i, g, k, \b_{i-1}.\{\G_i\}, \I_{\b_{i-1}}.\{0\})$ and $F(\S_i, g, k, \b_{i-1}.\{\G_i\}, \I_{\b_{i-1}}.\{1\})$ for all $1 \leq g \leq G$ , $1 \leq k \leq K$, $\I_{\b_{i-1}} \in \{0,1\}^{B_{i-1}}$, $i \in \{1,\dots,M\}$. 
Thus the number of values we shall need to store is at most $MGK 2^{B^*}$, which can be simplified to $\bigO(M^2 K G)$ using $2^{B^*} = \bigO(M)$ (due to our graph exploration algorithm). \\ 

Our algorithm for backtracking is as follows: 
We start from the $M$-th node and work backwards, determining the number of elements selected from each group. 
For the $M$-th group, we look at the optimal value for $G$ groups and $K$ elements, for the 2 cases when $\G_M$ is selected or unselected. 
The value which is the larger of these two forms our optimal solution, and thus tells us whether or not $\G_M$ is chosen in the optimal selection. 
If the optimal values stored at the $M-1$-th step involve other boundary nodes besides node $M$, we maximize over all selections of these boundary nodes, since we don't care about any particular nodes being selected in the optimal solution. 
We also remember the assignment of the indicator variables which allows us to obtain the largest value of $F$, since it tells us which nodes in $\b_{M-1}$ are included in the optimal solution. 
If we find that $\G_M$ is not chosen in the optimal selection, then we can ignore that group and simply find the optimal $G$-group, $K$-element selection on $M-1$ groups. 

If $\G_M$ is chosen, however, we must determine the number of elements that are selected from $\G_M$, after it is cleaned of elements from other selected boundary nodes. 
We do this by repeating step 2, case (b), of the value update rule, and noting the optimal value of the parameter $\ell$ which is used in computing a given value on the LHS. 
For the $M$-th group, we are specifically concerned with the optimal value for $g = G$ and $k = K$ on the LHS, and we must choose the indicator variables to maximize the value of F. 
Noting the value of $\ell$ which gives the optimum on the RHS tells us the number of elements chosen from $\G_M$ in the optimal selection, after cleaning any overlapping selected boundary nodes. 
Suppose this value is $\ell_1$. 
Then we now need to solve a smaller problem - find the optimal selection of $G-1$ groups and $K - \ell_1$ elements from groups $1, \dots, M-1$, with the nodes in $\b_{M-1}$ fixed to the maximizing value of their indicator variables. 
Clearly, we can repeat the above procedure on this smaller problem, and hence recursively determine the entire optimal selection. \\

It can be verified that the running time of the above algorithm is somewhat smaller than the update rule. 
Thus, the overall expression for time complexity is unchanged even when we account for backtracking.


\subsection {Graph Exploration Rule}
We determine the order with which the nodes are picked by a value associated to each subtree of the graph, which we call the $D$-value.
In the following, we describe how it is computed, how it depends logarithmically on the number of nodes in the graph and how the number of boundary nodes is bounded by the $D$-value. 

We start with some definitions.

\begin{definition}
Given a graph ${\mathbb G} = ({\cal V, \cal E})$, and an `explored set' $\cal S \subseteq \cal V$ of its nodes, a node $v \in \cal V$ is said to be a {\em boundary node} in ${\mathbb G}$ with respect to ${\cal S}$ if $v \in \cal S$ and $\exists u \in {\cal V}$ such that $u \notin {\cal S}$ and $(u,v) \in {\cal E}$.
\end{definition}

\begin{definition}
A {\em rooted tree graph} ${\mathbb G} = ({\cal V, \cal E}, r)$ is a tree graph with vertices ${\cal V}$ and edges ${\cal E}$, and a specific node $r \in {\cal V}$ designated as the root. 
\end{definition}

\begin{definition}
The {\em rooted subtrees} of a rooted tree graph ${\mathbb G} = ({\cal V, \cal E}, r)$ are the $d$ rooted tree graphs obtained as components when the root of ${\mathbb G}$ is deleted. The roots of the subtrees are the unique nodes which were adjacent to $r$ in ${\mathbb G}$. Note that $d$ is the degree of $r$ in ${\mathbb G}$.
\end{definition}

\begin{definition}
The {\em D-value} of a rooted tree graph is a non-negative integer associated with the graph. We will define the D-value algorithmically later.
\end{definition} 

{\bf Exploration Rule:}
Given a rooted tree graph $\mathbb G$, we first order all rooted subtrees with respect to the the $D$-value, so that $D_1 \geq \ldots \geq D_R$ for subtrees $T_1, T_2, \ldots, T_R$.
We then pick the subtrees in the order $\{T_1,\text{root}, T_2, \ldots, T_R\}$ and recurse until the explored subtree has only one node, see Fig.~\ref{fig:node_pick}.

\begin{figure}
\centering
\includegraphics[width=0.7\textwidth]{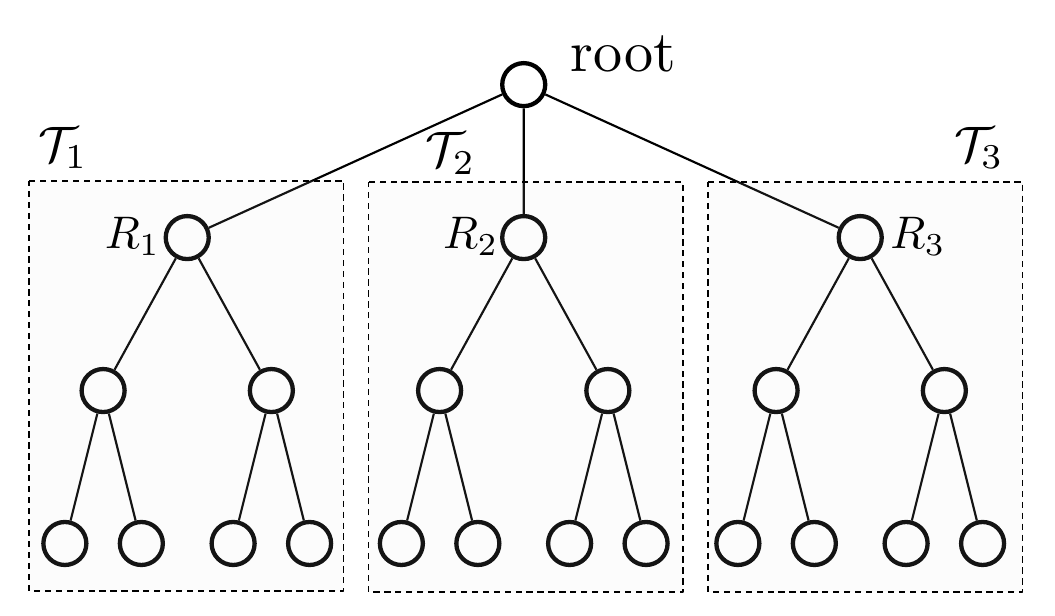}
\caption{\label{fig:node_pick} Graph Exploration Rule: explore nodes in the order $\mathcal{T}_1, \text{root}, \mathcal{T}_2, \mathcal{T}_3$ where $D_1 \geq D_2 \geq D_3$. For the subtree $\mathcal{T}_1$, the node connected to $\texttt{root}$ should be considered the root of $\mathcal{T}_1$, which we denote by $R_1$; similarly for the other subtrees.}
\end{figure}

{\bf Computing $D$-values:}
The procedure for computing the $D$-values is also recursive. 
If the tree has only one node, $D = 1$.
Now, assume the $R$ subtrees at a node $Q$ have values $D_1 \geq \ldots \geq D_R$. 
Then, $D(Q) = \max( D_1, D_2+1 )$. 
In case there is no second subtree, $D(Q) = D_1$.
We then have the following bound on the $D$-values.

\begin{lemma}
\label{lem:dvalue}
The $D$-value of a rooted tree graph is logarithmic in the number of nodes, i.e. $D(G) \leq \log_2(M) +1$.
\end{lemma}
\begin{proof}
Let $D$ be a positive integer and $N(D)$ be the minimum number of nodes that a rooted tree must have in order to have $D$-value of D. 
We prove by induction that 
\begin{equation}
\label{eq:dvalue}
N(D) \geq 2^{D-1}.
\end{equation} 

{\em Base case: $D = 2$}. A tree with only one node will have a $D$-value of 1. So to have a $D$-value of 2, we require a graph with at least 2 nodes. Hence \eqref{eq:dvalue} is satisfied.

{\em Inductive case: $D > 2$}. Let $\mathcal{T}$ be a smallest (i.e. minimum node) rooted tree graph whose $D$-value is equal to $D$. 
Spread out $\mathcal{T}$ in the form of root and subtrees. 
Let the subtrees be $\mathcal{T}_1, \mathcal{T}_2, \ldots, \mathcal{T}_k$, with corresponding $D$-values $D_1, D_2, \ldots, D_k$. 
Without loss of generality, assume that $D_1 \geq D_2 \geq \ldots \geq D_k$. 
By definition,  $D(\mathcal{T}) = \max(D_1, D_2+1)$. 

By our assumption, $\mathcal{T}$ is a minimum-node graph with $D$-value equal to $D$, hence we cannot have $D_1 =  D(G) = D$, since that would give us a smaller rooted tree graph ($\mathcal{T}_1$) with a $D$-value of $D$. 
This means that $D_1 < D$, and since $D = max(D_1, D_2 + 1)$, hence $D_2 + 1 = D$, i.e. $D_2 = D - 1$. 
Since $D_1 \geq D_2= D-1$ and $D_1 < D$, then $D_1 = D-1 = D_2$. 
Thus, the graph $\mathcal{T}$ has 2 subtrees ($\mathcal{T}_1$ and $\mathcal{T}_2$), with $D$-values of $D-1$ each.
By definition, any rooted subtree with a $D$-value of $D-1$ must have at least $N(D-1)$ nodes. 
By our induction hypothesis, $N(D-1)  \geq 2^{D-2}$ . 
Therefore, $\mathcal{T}$ has at least $2\times 2^{D-2} = 2^{D-1}$  nodes. 
But since $\mathcal{T}$ was the smallest rooted tree graph with $D$-value of $D$, this means that $N(D) \geq 2^{D-1}$, as required.
\end{proof}

We now link the number of boundary nodes visited by the algorithm to the $D$-value of the intersection graph.

\begin{lemma}
\label{lem:boundary}
The total number of boundary nodes encountered by the graph exploration algorithm cannot exceed the $D$-value of the graph.
\end{lemma}
\begin{proof}

Let $\mathcal{T}$ be the given rooted tree graph, with $M$ nodes. 
We shall consider the number of boundary nodes when there is a {\em ghost node} connected to the root node. 
The ghost node is a hypothetical node which is not really a part of the graph, but still makes adjacent explored nodes count as boundary nodes.
The ghost node captures the fact when we are running the algorithm recursively on a subtree, there will be an additional (potentially unexplored) node connected to the root of the subtree, which may lead to the root being counted as a boundary node.
Let $B^*(\mathcal{T})$ denote the maximum number of boundary nodes encountered on $\mathcal{T}$ when we pick nodes according to our algorithm, and let $B^*_G(\mathcal{T})$ represent the same when we also have the ghost node. 
Clearly, $B^*_G(\mathcal{T}) \geq B^*(\mathcal{T})$, hence it is enough to prove the following:
\begin{equation}
\label{eq:boundary}
B_G^*(\mathcal{T}) \leq D(\mathcal{T}).
\end{equation}
We prove this by strong induction on $M$.

{\em Base Case.} Suppose the rooted tree graph $\mathcal{T}$ has only $1$ node. 
Then the maximum number of boundary nodes encountered is obviously $1$, which is equal to the $D$-value of the graph (by definition). 
Hence $B^*_G(\mathcal{T}) \leq D(\mathcal{T})$.

{\em Inductive Case.} 
When the graph $\mathcal{T}$ consists of $M$ nodes, $M > 1$, consider the graph to be spread out in the form of root and subtrees. 
Compute the $D$-values for each rooted subtree, where w.l.o.g., $D_1 \geq D_2 \geq \ldots D_k$. 
Let $\mathcal{T}_1, \mathcal{T}_2, \ldots, \mathcal{T}_k$ be the corresponding subtrees. 
By definition, our algorithm explores nodes in the sequence:  $\mathcal{T}_1, \text{root}, \mathcal{T}_2, \mathcal{T}_3, \ldots \mathcal{T}_k$.

Since each subtree has strictly fewer than $M$ nodes, each subtree satisfies \eqref{eq:boundary} by the induction hypothesis. 
Also, notice that when exploring the subtree $\mathcal{T}_1$ of $\mathcal{T}$, the number of boundary nodes encountered is less than or equal to the number of boundary nodes encountered when exploring $\mathcal{T}_1$ as a standalone rooted-tree-graph, with a ghost node connected to its root. 
By definition, this is exactly equal to $B^*_G(\mathcal{T}_1)$, which by our induction hypothesis is bounded by $D_1$. Therefore, the number of boundary nodes encountered while exploring $\mathcal{T}_1$ in $\mathcal{T}$ cannot exceed $D_1$. Once we are finished with $\mathcal{T}_1$, we pick the root, so the total number of boundary nodes is $1$. 
We now proceed to pick $\mathcal{T}_2$. 
By a similar argument, the maximum number of boundary nodes in $\mathcal{T}_2$ at any point cannot exceed the number of boundary nodes encountered while exploring $\mathcal{T}_2$ as a standalone graph with attached ghost node. In addition, the root of $\mathcal{T}$ can contribute at most 1 additional boundary node (In fact, the ghost node for $\mathcal{T}$ ensures that the root, once picked, will always contribute an additional boundary node). 
Therefore, the total number of boundary nodes in $\mathcal{T}$ while exploring $\mathcal{T}_2$ is at most $D_2+1$. 
Similar arguments hold for all other subtrees --- the maximum number of boundary nodes while exploring the $k$-th subtree will be at most $D_k+1$, which is upper bounded by $D_2+1$.

Therefore, the maximum number of boundary nodes encountered at any step while exploring $\mathcal{T}$ is
$B_G^*(\mathcal{T}) \leq \max(D_1,D_2+1)$. 
By definition, $D(\mathcal{T}) = \max(D_1,D_2+1)$. 
Therefore $B_G^*(\mathcal{T}) \leq D(\mathcal{T})$.
\end{proof}

Combining Lemmas \ref{lem:dvalue} and \ref{lem:boundary}, we have the following result.

\begin{lemma}
\label{prop:boundary}
The maximum number of boundary nodes at any step of the algorithm is logarithmic in the number of nodes, i.e., $B \leq \log_2(M) + 1$.
\end{lemma}

The previous lemma establishes the polynomial time complexity of the dynamic program for solving the generalized integer problem \eqref{eq:GWMC}.

We shall now prove that the exploration rule itself requires minimal computation. This will justify our earlier claim that the running time is determined solely by the value update rule.

\begin{lemma}
The running time of the graph exploration rule is $\bigO(M)$ for an $M$-node graph. 
\end{lemma}
\begin{proof}
The exploration rule can be algorithmically run in two loops. 
In the first, we compute all $D$-values of all the required subtrees in the graph. In the second loop, we find the exploration ordering using these $D$-values. Note that the subtrees encountered by our recursive $D$-value computing algorithm are exactly the same set of subtrees encountered by our exploration rule, which makes it possible to compute all the required $D$-values in a single loop. 

For computing $D$-values at a particular node, we use the formula $D = \max(D_1, D_2 + 1)$, where $D_1 \geq D_2 \geq D_3, D_4, \dots, D_k$. Thus, we need to find the largest and second largest $D$-values among the subtrees. For a node with $d$ children, this takes $\bigO(d)$ time. Since the values $D_1, D_2, \dots, D_k$ are obtained recursively, this is the only computation which needs to be performed at the current node. Hence, the total time required is proportional to $\sum_{v \in \mathcal{V}}\max(d(v),1) \leq 2M$, where $d(v)$ represents the number of children that node $v$ has. Hence, this loop runs in $\bigO(M)$ time.

Obtaining the exploration order is similar. We only need to find the subtree with the largest $D$-value at the current node, so that we can pick the subtrees in the right order. This takes $\bigO(d)$ time for a node with $d$ children, and hence $\bigO(M)$ time for all nodes. 
Since both the above steps are $\bigO(M)$, the graph exploration rule itself runs in $\bigO(M)$, i.e., linear time.
\end{proof}

\begin{theorem}
The proposed dynamic program solves the Weighted Maximum Coverage problem with an additional constraint on element sparsity for acyclic group structures. 
Its time complexity is $\bigO(M^2GK^2 )$, where $M$ is the number of groups, $G$ is the group sparsity budget and $K$ is the element sparsity budget.
\end{theorem}


\section{Dynamical programming for solving the hierarchical signal approximation problem \eqref{eq:hier}}
\label{sec:dp_hier}

Here we describe the dynamic program for solving the hierarchical signal approximation problem \eqref{eq:hier} and show that its time complexity is $\bigO(NK^2D)$, for general trees with maximum degree $D$ and $\bigO(NKD)$ for $D$-regular trees. Furthermore, its space complexity for $D$-regular trees is $\bigO (N \log_DK)$.

\subsection{Problem description}

Problem \eqref{eq:hier} can be equivalently rephrased as the following optimization problem. \\

{\em Rooted-Connected Subtree Problem:} Given a rooted tree $\mathcal{T}$ with each node having at most $D$ children, a non-negative real number (weight) assigned to every node and a positive integer $K$, choose a subset of its nodes forming a rooted-connected subtree that maximizes the sum of weights of the chosen elements, such that the number of selected nodes does not exceed $K$. \\
\indent In our case, \eqref{eq:hier}, the weight of a node is the square of the value of the component of the signal associated to that node. The proposed algorithm leverages the optimal substructure of the problem.

\subsection{Optimal substructure} 

Suppose that a particular node X belongs to the optimal $K$-node rooted-connected subtree. 
Consider the subtree $\mathcal{T}_{X,d}$ obtained by choosing X, $d$ of its children ($1\leq d \leq D$) and all descendants of these children. 
Consider the set of nodes $\mathcal{S}$ consisting of all the nodes of $\mathcal{T}_{X,d}$ which are also present in the optimal $K$-node rooted-connected subtree. 
Suppose there are $L$ nodes in $\mathcal{S}$. 
Then the nodes in $\mathcal{S}$ form the optimal $L$-node rooted-connected subtree at X, for the subgraph $\mathcal{T}_{X,d}$. 
See Fig.~\ref{fig:DP_hier_subproblem} for an example.

\begin{figure}[t]
\centering
\includegraphics[width=0.7\textwidth]{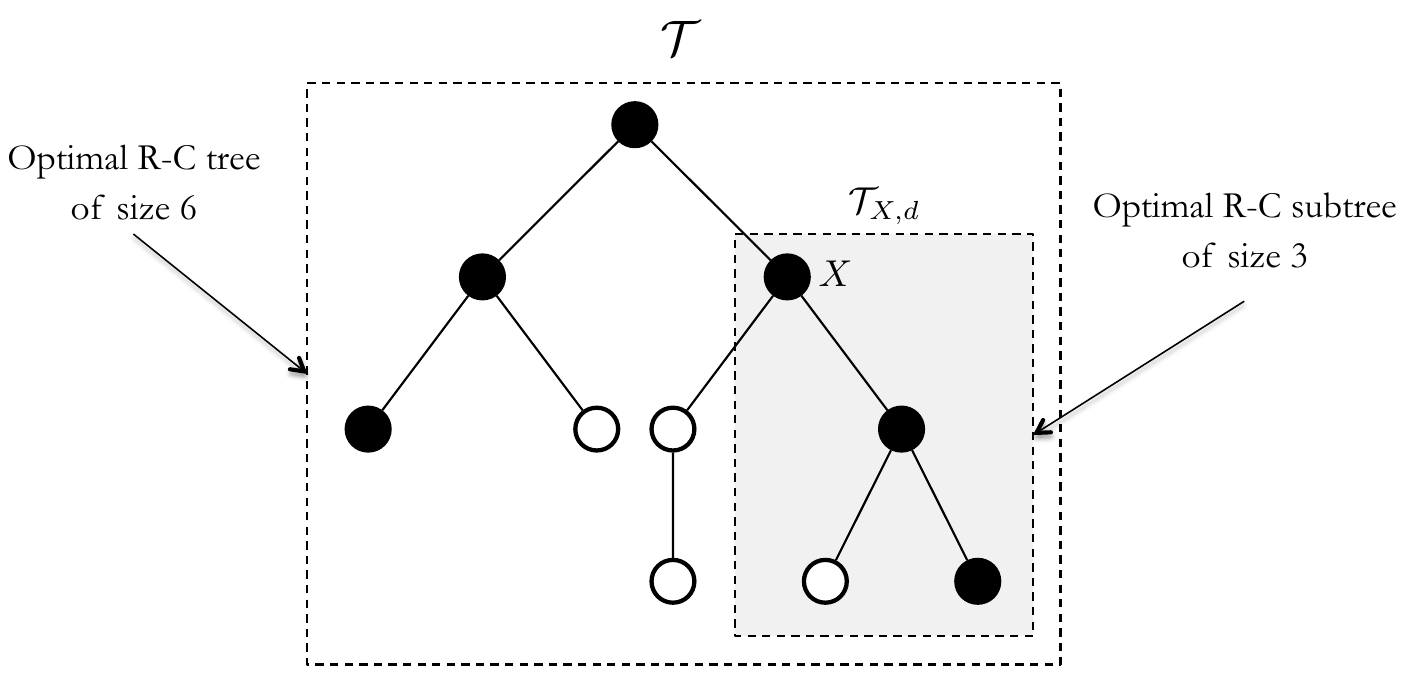}
\caption{Example of a nested subproblem in hierarchical groups model}
\label{fig:DP_hier_subproblem}
\end{figure}

\subsection{Dynamic Programming method.}  For every node X, we store the weight of the optimal $k$-node rooted-connected subtree at X, using only the nodes in the $d$ rightmost children of X and their descendants, for each $k$ and $d$ such that $1 \leq k \leq K$ and $1 \leq d \leq D$. 
We define a function $F(X,k,d)$, to store these optimal values. 
We start from the leaf nodes and move upwards, for each node assessing all its subtrees from right to left, eventually covering the entire tree. 
At the end, the optimal value will be given by $F(\text{root},K,D)$, that is the value of the best K-node rooted connected subtree of the root considering all its descendants.

{\it Base Case.} For every leaf node X and for all $1\leq k \leq K$ and $ 1\leq d \leq D$, we set $F(X,k,d) = \text{Weight}(X)$. 

{\it Inductive Case.} By induction, for every non-leaf node X, all the F-values are known for the descendants of X. 
Let $X_1, X_2, \ldots X_d$ be the $d$ children of X in the right-to-left order, where $1 \leq d \leq D$. 
Then, we compute the F-values of X using the following update rules. \\

{\bf Value Update Rule:}
\begin{enumerate}
\item For all $1\leq k \leq K$
$$
F(X,k,1) = \text{Weight}(X)+ F(X_1,k-1,D) \; .
$$
The optimal value for choosing a $k$-node subtree rooted at $X$, when only the rightmost child $X_1$ is allowed, equals the weight of $X$ itself (since $X$ must be chosen), plus the optimal value for choosing a rooted connected subtree with $k-1$ nodes from the rightmost child $X_1$.\\

\item For all $1 \leq k \leq K$ and  $1 < i \leq d$
$$
F(X,k,i) = \max_{0 \leq \ell \leq k-1} \left \{ F(X,k-\ell,i-1)  + F(X_i  ,\ell,D) \right \} \; .
$$
For choosing the best $k$-node rooted connected subtree from the rightmost $i$ children, choose a positive integer $\ell \leq k$, pick the best $k-\ell$-node subtree at $X$ by including the rightmost $i -1$ children and pick the remaining $\ell$ nodes from the subtree of the $i$th child. 
We then take the maximum over all $\ell$, $0 \leq \ell \leq k-1$ (since at least $1$ node must be chosen from the rightmost $i-1$ nodes, this node will be the root). \\

\item For all $1\leq k \leq K$ and $d < i \leq D$
$$
F(X,k,i)=F(X,k,d) \; .
$$
For convenience, when a node has only $d$ children, where $d$ is strictly less than $D$, we set F-values for cases involving more than $d$ children equal to the value for $d$ children.
\end{enumerate}

\subsection{Running Time}

\newtheorem*{th:dp_hier_gen}{Theorem~\ref{th:dp_hier_gen}}
\begin{th:dp_hier_gen}
Given a hierarchical group structure $\mathfrak{G}$, the time complexity of the dynamic programming algorithm is $\mathcal{O}(NK^2D)$, where $D$ is maximum number of children of a node in the tree.
\end{th:dp_hier_gen}
\begin{proof}
The main cost of the dynamic program is evaluating the second value update rule. 
Let $X_i$ be the i-th node in the tree, $d_i$ the number of its children $X_{i,1}, \ldots, X_{i,d_i}$. 
Let also $K_i$ be the cardinality of the tree that has $X_i$ as root and $K_{i,j}$ be the cardinality of the tree that has $X_{i,j}$ as root for $1 \leq i \leq N$ and $1 \leq j \leq d_i$. 
Given $X_i$, evaluating $F(X_i, k, j)$ for $1 \leq k \leq \min(K, K_i)$ and $1 \leq j \leq d_i$ requires $\min(k,K_{i,j})$ operations.
Therefore, overall we need to compute 
\begin{align*}
& \sum_{i=1}^N \sum_{k=1}^{\min(K,K_i)} \sum_{j=1}^{d_i} \min(k,K_{i,j}) \\
& \leq \sum_{i=1}^N \sum_{k=1}^{\min(K,K_i)} \sum_{j=1}^{D} \min(k,K_{i,j}) \\ 
& \leq \sum_{i=1}^N \sum_{k=1}^K Dk \\ 
& = \bigO ( NK^2D)
\end{align*}
values, each of which requires a simple operation.
\end{proof}

By leveraging the special structure of $D$-regular trees, it is possible to prove that the complexity of the dynamic program is linear in $K$.

\newtheorem*{prop:dp_hier_reg}{Proposition~\ref{prop:dp_hier_reg}}
\begin{prop:dp_hier_reg}
The time complexity of the dynamic program for $D$-regular trees is $\mathcal{O}(KDN)$.
\end{prop:dp_hier_reg}
\begin{proof}
The proof follows the arguments in \cite{cartis2013exact}.
Suppose there are $J$ levels in our tree, hence the maximum number of nodes that can be selected for a sub-tree with root in level $j$ is $S_j = 1+D+D^2+ \cdots+D^{J-j} = \frac{D^{J-j+1}-1}{D-1}$ where $j \in \{1,2, \ldots, J\}$.
At each step, the dynamic program considers selecting at most $K$ elements to form a sub-tree. 
Hence for a sub-tree with root at level $j$, we can select a maximum number of $ \mathcal{O}(l(j))=\mathcal{O}(\min(K,S_j))=\mathcal{O}(\min(K,D^{J-j}))$ for $D\geq 3$ and $\mathcal{O}(l(j))=\mathcal{O}(\min(K,D^{J-j+1}))$ for $D=2$. 
Note that we do not require any computation for level $J$. 
The update step of $F(X,k,i)=\max_{0\leq \ell \leq \min(k,l(j+1))}{F(X,k-\ell,i-1)+F(X_i,\ell,D)}$ then requires $\mathcal{O}(\min(k,l(j+1)))$ operations and for X in level $j$ this needs to be calculated $\forall 1 \leq k \leq l(j)$ and $1\leq i \leq D$. 
This leads to at most $\mathcal{O}\left (D\sum_{k=1}^{l(j)} \min(k,l(j+1)) \right )$ operations. 
By considering that at level $j$ there are at most $D^{j-1}$ nodes, the total number of operations can be written as
\begin{equation}
\label{eq:operations_sum}
\mathcal{O}\left ( \sum_{j=1}^{J} D^{j-1}D\sum_{k=1}^{l(j)} \min(k,l(j+1)) \right )
\end{equation}

Let $j'$ be such that $K \leq D^{J-j}$ for all $j < j'$. 
We then have $ j' = J - \lfloor \log_DK \rfloor$ and $\min(K, D^{J-j}) = K$ for $j < j'$ and $\min(K, D^{J-j}) = D^{J-j}$ for $j \geq j'$.
Hence we can break \eqref{eq:operations_sum} into

\begin{align*}
& \bigO \left ( \sum_{j=1}^{j'-1}D^j\sum_{k=1}^{K} k + \sum_{j=j'}^{J}D^j\sum_{k=1}^{D^{J-j}} \min(k,D^{J-j-1}) \right )\\
& = \bigO \left (\sum_{j=1}^{j'-1}D^jK^2 + \sum_{j=j'}^{J} D^j \left ( \sum_{k=1}^{D^{J-j-1}} k + \sum_{k=D^{J-j-1}+1}^{D^{J-j}} D^{J-j-1} \right ) \right ) \\
& = \bigO \left (\sum_{j=1}^{j'-1}D^jK^2 + \sum_{j=j'}^{J} D^j \left ( D^{2J - 2j - 2} + DD^{J-j-1} \right ) \right )\\
& \leq \bigO \left ( K^2 \frac{D^{j'}}{D-1} + D^{2J - 2} \sum_{j=j'}^{J} D^{-j} \right ) \\ 
& \leq \bigO \left ( K^2 \frac{D^{j'}}{D-1} + D^{2J - 2} \frac{D^{-j'}}{1-D^{-1}}) \right ) \\
& = \bigO \left ( K^2 \frac{D^{J-\lfloor \log_DK \rfloor}}{D-1} + K\frac{D^{J - 2}}{1-D^{-1}}) \right ) \\
& \leq \bigO \left ( K \frac{D^{J+1}}{D-1} + K\frac{D^{J - 2}}{1-D^{-1}}) \right ) \\
& = \bigO \left ( K \frac{D^{J+1}}{D-1} + K\frac{D^{J - 1}}{D-1}) \right ) \\
& = \bigO \left ( K \frac{D^{J+1}}{D-1} \right ) \; .
\end{align*}


For $D$-regular trees (with $D \geq 3$), we have $N=\dfrac{D^J-1}{D-1}\approx\dfrac{D^J}{D-1}$ 
, so that the time complexity will be $\mathcal{O} (KDN)$. 
When $D=2$, we can follow the same steps to show that the complexity is $\mathcal{O}(KD^2N)$.
But for small values of $D$, $\mathcal{O}(ND^2K)=\mathcal{O}(NDK)$. 
Hence we can say that the overall complexity is $\mathcal{O} (NDK)$.
\end{proof}

\subsection{Space Complexity}

\newtheorem*{prop:dp_hier_reg_space}{Proposition~\ref{prop:dp_hier_reg_space}}
\begin{prop:dp_hier_reg_space}
The memory complexity of the dynamic program for D-regular trees is $\mathcal{O}(N\log_DK)$ for our implementation.
\end{prop:dp_hier_reg_space}

\begin{proof}
Suppose there are total of $J \geq 1$ levels in our tree. 
Hence the maximum number of nodes that can be selected for a sub-tree with root in level $j$ is  $\ell(j)=\min(K,1+D+D^2+..+D^{J-j})=\min(K,\frac{D^{J+1-j}-1}{D-1})$ or  $\mathcal{O}(\ell(j)) =\mathcal{O}(\min(K,D^{J-j}))$, where $j \in \{1,2, \ldots,J\}$. 
Let $N_j = D^{j-1}$ be the number of nodes at level $j$.

In order to recover the optimal selection of nodes from the dynamic program, we use a standard backtracking procedure: for each node, we store the number of selected nodes in each of its subtrees ($D$ numbers) in the optimal selection for $1 \leq k \leq \ell(j)$.

Hence, the total memory required is
\begin{equation}
\bigO \left ( \sum_{j=1}^{J} D\ell(j) N_j \right ) = \bigO \left ( \sum_{j=1}^{J} \min(K,D^{J-j})D^j \right )
\end{equation}

Let $j'$ be such that $K \leq D^{J-j}$ for all $j < j'$. 
We then have $ j' = J - \lfloor \log_DK \rfloor$ and $\min(K, D^{J-j}) = K$ for $j < j'$ and $\min(K, D^{J-j}) = D^{J-j}$ for $j \geq j'$.

We can now write

\begin{align}
& = \bigO \left ( \sum_{j=1}^{j'-1} K D^j + \sum_{j=j'}^{J} D^{J-j}*D^j \right )\\
& = \bigO \left ( K \sum_{j=1}^{j'-1} D^j + \sum_{j=j'}^{J} D^J \right )\\
& = \bigO \left ( K D^{j'}) + D^J (J-j') \right )  \\
& \leq \bigO \left ( K D^{J+1-\log_DK} + D^J \log_DK \right ) \\
& = \bigO \left ( D^{J+1} + D^J \log_DK \right ) \\
& = \bigO \left ( N \log_DK \right ) \text{ where } N = \mathcal{O} (D^{J}) 
\end{align}
\end{proof}

\section*{Acknowledgements}
We would like to sincerely thank the anonymous reviewers for their detailed and constructive observations and criticisms.
We also thank Nikhil Rao for providing the code for block signal recovery with the Latent Group Lasso approach.

\bibliographystyle{IEEEtran}
\bibliography{groupsparse}

\begin{IEEEbiographynophoto}{Luca Baldassarre}
received his M.Sc. in Physics in 2006 and his Ph.D. in Machine Learning in 2010 at the University of Genoa, Italy. 
He then joined the Computer Science Department of University College London, UK,  to work with Prof. Massimiliano Pontil on structured sparsity models for machine learning and convex optimization.
Currently he is with the LIONS of Prof. Volkan Cevher at the \'Ecole Polytechnique F\'ed\'erale de Lausanne (EPFL), Switzerland. 
His research interests include model-based machine learning and compressive sensing and optimization.
\end{IEEEbiographynophoto}

\begin{IEEEbiographynophoto}{Nirav Bhan}
is currently a second year graduate student in EECS at Massachusetts Institute of Technology. He is member of the Laboratory of Information and Decision Systems (LIDS). His interests are in optimization, machine learning, graphical models, and applying mathematics to solve problems. Prior to being a graduate student, Nirav obtained a B.Tech degree in Electrical Engineering along with a Minor in Computer Science, from the Indian Institute of Technology-Bombay. He has worked as a research assistant at LIONS, EPFL, during the period of May to July, 2012.
\end{IEEEbiographynophoto}

\begin{IEEEbiographynophoto}{Volkan Cevher}
received his BSc degree (valedictorian) in Electrical Engineering from Bilkent University in 1999, and his PhD degree in Electrical and Computer Engineering from Georgia Institute of Technology in 2005. He held Research Scientist positions at University of Maryland, College Park during 2006-2007 and at Rice University during 2008-2009. Currently, he is an Assistant Professor at \'Ecole Polytechnique F\'ed\'erale de Lausanne and a Faculty Fellow at Rice University. He coauthored (with C. Hegde and M. Duarte) the Best Student Paper at the 2009 International Workshop on Signal Processing with Adaptive Sparse Structured Representations (SPARS). In 2011, he received an ERC Junior award. His research interests include signal processing theory, machine learning, graphical models, and information theory.
\end{IEEEbiographynophoto}

\begin{IEEEbiographynophoto}{Anastasios Kyrillidis}
received his 5-year diploma and M.Sc. in Electronic and Computer Engineering from Technical University of Crete in 2008 and 2010, respectively, and his PhD in Electrical and Computer Engineering from \'Ecole Polytechnique F\'ed\'erale de Lausanne in 2014. Currently, he is a Simons Fellowship PostDoc researcher at the WNCG Group at University of Texas at Austin. His research interests include convex and non-convex optimization, machine learning and high-dimensional data analysis and statistics.
\end{IEEEbiographynophoto}

\begin{IEEEbiographynophoto}{Siddhartha Satpathi}
 is expected to graduate in July 2015, with a B.Tech+M.Tech degree in Electrical Engineering and a minor in Computer Science. He has worked as a research intern at EPFL, Switzerland, during the period of May to July, 2013. His interests are in compressive sensing, greedy algorithms and energy harvesting communication systems.
 \end{IEEEbiographynophoto}




\end{document}